\RequirePackage{fix-cm}
\documentclass[smallcondensed]{svjour3}     % onecolumn (ditto)
\smartqed  % flush right qed marks, e.g. at end of proof
\usepackage{graphicx}
\usepackage{amsfonts} 
\usepackage{amsmath}
\usepackage{stmaryrd}
\usepackage{subfigure}
\usepackage{footnote}
\usepackage{makecell}
\usepackage{multirow}
\usepackage{color}
\usepackage{algorithm}
\usepackage{algorithmic}
\usepackage{natbib}

\usepackage{tikz}
\usetikzlibrary{arrows.meta, shapes, positioning, matrix, calc}
\definecolor{dpink}{RGB}{253,0,216}
\definecolor{c1}{RGB}{0,176,80}
\definecolor{c2}{RGB}{255,0,0}
\definecolor{c3}{RGB}{0,204,255}
\definecolor{ct}{RGB}{112,48,160}
\definecolor{cr}{RGB}{31,73,125}
\tikzstyle{rec} = [rounded corners, rectangle, draw=cr]
\tikzstyle{z1} = [diamond, draw=none, fill=c1, scale=0.8]
\tikzstyle{z2} = [circle, draw=none, fill=c2, scale=0.8]
\tikzstyle{z3} = [rectangle, draw=none, fill=c3, scale=0.8]
\tikzstyle{zt} = [regular polygon, regular polygon sides=3, draw=ct, fill=ct, scale=0.48]

\graphicspath{{./figure/}}
\begin{document}

\title{Cost-Sensitive Label Embedding for Multi-Label Classification}
\subtitle{}
\author{Kuan-Hao Huang         \and
        Hsuan-Tien Lin
}


\institute{Kuan-Hao Huang, Hsuan-Tien Lin\at
              CSIE Department, National Taiwan University, Taipei, Taiwan \\
              \email{\{r03922062, htlin\}@csie.ntu.edu.tw}
}

\date{Received: date / Accepted: date}

\maketitle

\begin{abstract}
Label embedding (LE) is an important family of multi-label classification algorithms that digest the label information jointly for better performance. 
Different real-world applications evaluate performance by different cost functions of interest.
Current LE algorithms often aim to optimize one specific cost function, but they can suffer from bad performance with respect to other cost functions. 
In this paper, we resolve the performance issue by proposing a novel cost-sensitive LE algorithm that takes the cost function of interest into account.
The proposed algorithm, cost-sensitive label embedding with multidimensional scaling (CLEMS), approximates the cost information with the distances of the embedded vectors by using the classic multidimensional scaling approach for manifold learning.
CLEMS is able to deal with both symmetric and asymmetric cost functions, and effectively makes cost-sensitive decisions by nearest-neighbor decoding within the embedded vectors.
We derive theoretical results that justify how CLEMS achieves the desired cost-sensitivity. Furthermore, extensive experimental results demonstrate that CLEMS is significantly better than a wide spectrum of existing LE algorithms and state-of-the-art cost-sensitive algorithms across different cost functions.
\end{abstract}

\section{Introduction}
The multi-label classification problem (MLC), which allows multiple labels
to be associated with each example, is an extension of the multi-class classification problem.
The MLC problem satisfies the demands of many real-world applications \citep{Carneiro2007image,Trohidis2008music,Barutcuoglu2006gene}.
Different applications usually need different criteria to evaluate the prediction performance of MLC algorithms.
Some popular criteria are Hamming loss, Rank loss, F1 score, and Accuracy score \citep{Tsoumakas2010score,Madjarov2012score2}.

Label embedding (LE) is an important family of MLC algorithms that jointly extract the information of all labels to improve the prediction performance.
LE algorithms automatically transform the original labels to an embedded space, which represents the hidden structure of the labels. After conducting learning within the embedded space, LE algorithms make more accurate predictions with the help of the hidden structure.

Existing LE algorithms can be grouped into two categories based on the dimension of the embedded space: label space dimension reduction (LSDR) and label space dimension expansion (LSDE).
LSDR algorithms \citep{Hsu2009cs,Kapoor2012bcs,Tai2012plst,Sun2011cca,Chen2012cplst,Yu2014leml,Lin2014faie,Balasubramanian2012landmark,Bi2013cssp,Bhatia2015sleec,Yeh2017c2ae} consider a low-dimensional embedded space for digesting the information between labels and conduct more effective learning.
In contrast to LSDR algorithms, LSDE algorithms \citep{Zhang2011ccaoc,Ferng2013ecc,Tsoumakas2011rakel} focus on a high-dimensional embedded space.
The additional dimensions can then be used to represent different angles of joint information between the labels to reach better performance.

While LE algorithms have become major tools for tackling the MLC problem,
most existing LE algorithms are designed to optimize only one or few specific criteria.
The algorithms may then suffer from bad performance with respect to other criteria.
Given that different applications demand different criteria, it is thus important to achieve cost (criterion) sensitivity to make MLC algorithms more realistic.
Cost-sensitive MLC (CSMLC) algorithms consider the criterion as an additional input, and take it into account either in the training or the predicting stage.
The additional input can then be used to guide the algorithm towards more realistic predictions.
CSMLC algorithms are attracting research attention in recent years \citep{Lo2011csrakel1,Lo2014csrakel2,Dembczynski2010pcc,Dembczynski2011pcc2,Li2014cft},
but to the best of our knowledge, there is no work on cost-sensitive label embedding (CSLE) algorithms yet.

In this paper, we study the design of CSLE algorithms, 
which take the intended criterion into account in the training stage to locate a cost-sensitive hidden structure in the embedded space.
The cost-sensitive hidden structure can then be used for more effective learning and more accurate predictions with respect to the criterion of interest.
Inspired by the finding that many of the existing LSDR algorithms can be viewed as linear manifold learning approaches, we propose to adopt manifold
learning for CSLE.
Nevertheless, to embed any general and possibly complicated criterion,
linear manifold learning may not be sophisticated enough. We thus start with
multidimensional scaling (MDS), one famous non-linear manifold learning approach, to propose a novel CSLE algorithm.
The proposed cost-sensitive label embedding with multidimensional scaling (CLEMS) algorithm embeds
the cost information within the distance measure of the embedded space. 
We further design a \textit{mirroring trick} for CLEMS to properly embed the possibly asymmetric criterion information within the symmetric distance measure.
We also design an efficient procedure that conquers the difficulty of making predictions through the non-linear cost-sensitive hidden structure.
Theoretical results justify that CLEMS achieves cost-sensitivity through learning in the MDS-embedded space.
Extensive empirical results demonstrate that CLEMS usually reaches better performance than leading LE algorithms across different criteria.
In addition, CLEMS also performs better than the state-of-the-art CSMLC algorithms \citep{Li2014cft,Dembczynski2010pcc,Dembczynski2011pcc2}.
The results suggest that CLEMS is a promising algorithm for CSMLC.

This paper is organized as follows.
Section~\ref{sec:csle} formalizes the CSLE problem and Section~\ref{sec:clems} illustrates the proposed algorithm along with theoretical justifications. 
We discuss the experimental results in Section~\ref{sec:exp} and conclude in Section~\ref{sec:final}.

%%%%%%%%%%%%%%%%%%%%%%%%%%%%%%%%%%%%%%%% Section Cost-sensitive Label Embedding %%%%%%%%%%%%%%%%%%%%%%%%%%%%%%%%%%%%%%%%
\section{Cost-sensitive label embedding}
\label{sec:csle}
In multi-label classification (MLC), we denote the feature vector of an instance by $\mathbf{x} \in \mathcal{X} \subseteq \mathbb{R}^{d}$ and denote the label vector by $\mathbf{y} \in \mathcal{Y} \subseteq \{ 0,1 \}^{K}$ where $\mathbf{y}[i]=1$ if and only if the instance is associated with the $i$-th label.
Given the training instances $\mathcal{D} = \{(\mathbf{x}^{(n)}, \mathbf{y}^{(n)}) \}_{n=1}^{N}$, 
the goal of MLC algorithms is to train a predictor $h\colon \mathcal{X} \rightarrow \mathcal{Y}$ from~$\mathcal{D}$ in the training stage,
with the expectation that for any unseen testing instance $(\mathbf{x}, \mathbf{y})$,
the prediction $\tilde{\mathbf{y}} = h(\mathbf{x})$ can be close to the ground truth~$\mathbf{y}$.

A simple criterion for evaluating the closeness between~$\mathbf{y}$ and $\tilde{\mathbf{y}}$ is \textit{Hamming loss}$(\mathbf{y}, \tilde{\mathbf{y}}) = \frac{1}{K}\sum_{i=1}^{K} \llbracket \mathbf{y}[i] \neq \tilde{\mathbf{y}}[i] \rrbracket$.
It is worth noting that Hamming loss separately evaluates each label component of~$\tilde{\mathbf{y}}$.
There are other criteria that jointly evaluate all the label components of $\tilde{\mathbf{y}}$, such as F1 score, Rank loss, 0/1 loss, and Accuracy score \citep{Tsoumakas2010score,Madjarov2012score2}.

Arguably the simplest algorithm for MLC is binary relevance (BR) \citep{Tsoumakas2007br}.
BR separately trains a binary classifier for each label without considering the information of other labels.
In contrast to BR, label embedding~(LE) is an important family of MLC algorithms that \textit{jointly} use the information of all labels to achieve better prediction performance.
LE algorithms try to identify the hidden structure behind the labels.
In the training stage, instead of training a predictor~$h$ directly, LE algorithms first embed each $K$-dimensional label vector~$\mathbf{y}^{(n)}$ as an $M$-dimensional embedded vector $\mathbf{z}^{(n)} \in \mathcal{Z}\subseteq \mathbb{R}^{M}$ by an embedding function $\Phi\colon \mathcal{Y} \rightarrow \mathcal{Z}$.
The embedded vector $\mathbf{z}^{(n)}$ can be viewed as the hidden structure that contains the information pertaining to all labels.
Then, the algorithms train a internal predictor $g\colon \mathcal{X} \rightarrow \mathcal{Z}$ from $\{(\mathbf{x}^{(n)}, \mathbf{z}^{(n)}) \}_{n=1}^{N}$.
In the predicting stage, for the testing instance $\mathbf{x}$, LE algorithms obtain the predicted embedded vector $\tilde{\mathbf{z}} = g(\mathbf{x})$ and use a decoding function $\Psi\colon \mathcal{Z} \rightarrow \mathcal{Y}$ to get the prediction~$\tilde{\mathbf{y}}$.
In other words, LE algorithms learn the predictor by $h = \Psi \circ g$.
Figure~\ref{fig:le} illustrates the flow of LE algorithms.

\begin{figure}[t]
	\centering
	\begin{tikzpicture}[minimum height=1em, minimum width=1em, line width=0.5pt, draw=cr]
		\node [] at (0.8,-0.4) {label space $\mathcal{Y}$};
		\node [] at (4.16,-0.4) {embedded space $\mathcal{Z}$};
		\node [] at (7.12,-0.4) {feature space $\mathcal{X}$};
		\node [z1, scale=0.7] at (0.48,0.4) (y1) {};
		\node [z2, scale=0.7] at (0.0,1.2) (y2) {};
		\node [z3, scale=0.7] at (1.2,0.0) (y3) {};
		\draw [] (y2) -- (0.0,0.0) -- (y3);
		\draw [] (y2) -- (1.2,1.2) -- (y3);
		\draw [] (y2) -- (0.48,1.6) -- (1.68,1.6) -- (1.2,1.2);
		\draw [] (1.68,1.6) -- (1.68,0.4) -- (y3);
		\draw [densely dotted] (y1) -- (0.0,0.0);
		\draw [densely dotted] (y1) -- (0.48,1.6);
		\draw [densely dotted] (y1) -- (1.68,0.4);
		\draw [-Latex, draw=cr] (2.0,0.96) to (2.96,0.96);
		\draw [-Latex, draw=cr] (2.96,0.64) to (2.0,0.64);
		\node [] at (2.4,1.28) {\large $\Phi$};
		\node [] at (2.56,0.32) {\large $\Psi$};
		\node [rec, minimum height=5em,  minimum width=5.2em] at (4.16, 0.8) [] {};
		\node [z1, scale=0.7] at (4.32,1.28) (z1) {};
		\node [z2, scale=0.7] at (3.68,0.56) (z2) {};
		\node [z3, scale=0.7] at (4.64,0.40) (z3) {};
		\draw [, -Latex, draw=cr] (6.24,0.8) to (5.28,0.8);
		\node [] at (5.84,0.48) {\large $g$};
		\draw [] (7.12, 0.8) ellipse (1.6em and 2.4em);
	\end{tikzpicture}
	\caption{Flow of label embedding}
	\label{fig:le}
\end{figure}
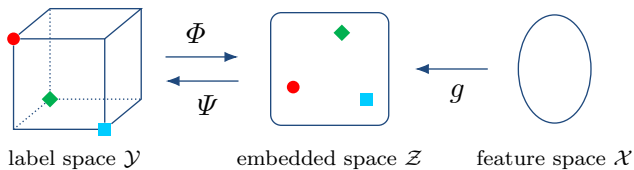

Existing LE algorithms can be grouped into two categories based on $M$ (the dimension of $\mathcal{Z}$) and $K$ (the dimension of~$\mathcal{Y}$).
LE algorithms that work with $M \leq K$ are termed as label space dimension reduction~(LSDR) algorithms. 
They consider a low-dimensional embedded space for digesting the information between labels and utilize different pairs of $(\Phi, \Psi)$ to conduct more effective learning.
Compressed sensing \citep{Hsu2009cs} and Bayesian compressed sensing \citep{Kapoor2012bcs} consider a random projection as~$\Phi$ and obtain $\Psi$ by solving an optimization problem per test instance.
Principal label space transformation \citep{Tai2012plst} considers~$\Phi$ calculated from an optimal linear projection of the label vectors and derives $\Psi$ accordingly.
Some other works also consider optimal linear projections as $\Phi$ but take feature vectors into account in the optimality criterion,
including canonical-correlation-analysis methods \citep{Sun2011cca}, 
conditional principal label space transformation \citep{Chen2012cplst},
low-rank empirical risk minimization for multi-label learning \citep{Yu2014leml},
and feature-aware implicit label space encoding \citep{Lin2014faie}.
Canonical-correlated autoencoder \citep{Yeh2017c2ae} extends the linear projection works by using neural networks instead.
Landmark selection method \citep{Balasubramanian2012landmark} and column subset selection \citep{Bi2013cssp} design $\Phi$ to select a subset of labels as embedded vectors and derive the corresponding $\Psi$.
Sparse local embeddings for extreme classification \citep{Bhatia2015sleec} trains a locally-linear projection as~$\Phi$ and constructs~$\Psi$ by nearest neighbors.
The smaller~$M$ in LSDR algorithms allows the internal predictor $g$ to be learned more efficiently and effectively.

Other LE algorithms work with $M > K$, which are called label space dimension expansion~(LSDE) algorithms.
Canonical-correlation-analysis output codes \citep{Zhang2011ccaoc} design $\Phi$ based on canonical correlation analysis to generate additional output codes to enhance the performance.
Error-correcting codes (ECC) algorithms \citep{Ferng2013ecc} utilize the encoding and decoding functions of standard error-correcting codes for communication as $\Phi$ and $\Psi$, respectively.
Random $k$-labelsets \citep{Tsoumakas2011rakel}, a popular algorithm for MLC, can be considered as an ECC-based algorithm with the repetition code \citep{Ferng2013ecc}.
LSDE algorithms use additional dimensions to represent different angles of joint information between the labels to reach the better performance.

To the best of our knowledge, all the existing LE algorithms above are designed for one or few specific criteria and may suffer from bad performance with respect to other criteria. 
For example, the optimality criterion within principal label space transformation \citep{Tai2012plst} is closely related to Hamming loss. For MLC data with very few non-zero $\mathbf{y}[i]$, which are commonly encountered in real-world applications, optimizing Hamming loss can easily results in all-zero predictions~$\tilde{\mathbf{y}}[i]$, which suffer from bad F1 score. 

MLC algorithms that take the evaluation criterion into account are called cost-sensitive MLC (CSMLC) algorithms and are attracting research attentions in recent years. 
CSMLC algorithms take the criterion as an additional input and consider it either in the training or the predicting stage.
For any given criterion, CSMLC algorithms can ideally make cost-sensitive predictions with respect to the criterion without extra efforts in algorithm design.
Generalized $k$-labelsets ensemble \citep{Lo2011csrakel1,Lo2014csrakel2} is extended from random $k$-labelsets \citep{Tsoumakas2011rakel} and digests the criterion by giving appropriate weights to labels.
The ensemble algorithm performs well for any weighted Hamming loss but cannot tackle more general criteria that jointly evaluate all the label components, such as F1 score.
Two CSMLC algorithms for arbitrary criterion are probabilistic classifier chain (PCC) \citep{Dembczynski2010pcc,Dembczynski2011pcc2} and condensed filter tree (CFT) \citep{Li2014cft}. 
PCC is based on estimating the probability of each label and making a Bayes-optimal inference for the evaluation criterion.
While PCC can in principle be used for any criterion, it may suffer from computational difficulty unless an efficient inference rule for the criterion is designed first.
CFT is based on converting the criterion as weights when learning each label. CFT conducts the weight-assignment in a more sophisticated manner than generalized $k$-labelsets ensemble does, and can hence work with
arbitrary criterion.
Both PCC and CFT are extended from classifier chain (CC) \citep{Read2011cc} and form a chain of labels to utilize the information of the earlier labels in the chain, but they cannot globally find the hidden structure of all labels like LE algorithms.

In this paper, we study the design of cost-sensitive label embedding (CSLE) algorithms that respect the criterion when calculating the embedding function~$\Phi$ and the decoding function~$\Psi$.
We take an initiative of studying CSLE algorithms, with the hope of achieving cost-sensitivity and finding the hidden structure at the same time.
More precisely, we take the following CSMLC setting \citep{Li2014cft}. 
Consider a cost function~$c( \mathbf{y}, \tilde{\mathbf{y}})$ which represents the penalty when the ground truth is $\mathbf{y}$ and the prediction is~$\tilde{\mathbf{y}}$.
We naturally assume that $c( \mathbf{y}, \tilde{\mathbf{y}}) \geq 0$, with value $0$ attained if and only if~$\mathbf{y}$ and $\tilde{\mathbf{y}}$ are the same.
Given training instances $\mathcal{D} = \{(\mathbf{x}^{(n)}, \mathbf{y}^{(n)}) \}_{n=1}^{N}$ and the cost function~$c(\cdot, \cdot)$, 
CSLE algorithms learn an embedding function $\Phi$, a decoding function~$\Psi$, and an internal predictor $g$, based on both the training instance~$\mathcal{D}$ and the cost function~$c(\cdot, \cdot)$.
The objective of CSLE algorithms is to  minimize the expected cost $c(\mathbf{y}, h(\mathbf{x}))$ for any unseen testing instance $(\mathbf{x}, \mathbf{y})$, where $h = \Psi \circ g$.

%%%%%%%%%%%%%%%%%%%%%%%%%%%%%%%%%%%%%%%% Proposed Algorithm %%%%%%%%%%%%%%%%%%%%%%%%%%%%%%%%%%%%%%%%
\section{Proposed algorithm}
\label{sec:clems}
We first discuss the difficulties of directly extending state-of-the-art LE algorithms for CSLE.
In particular, the decoding function $\Psi$ of many existing algorithms, such as conditional principal label space transformation \citep{Chen2012cplst} and feature-aware implicit label space encoding \citep{Lin2014faie}, are derived from $\Phi$ and can be divided into two steps.
The first step is using some $\psi \colon \mathcal{Z} \rightarrow \mathbb{R}^K$ that corresponds to~$\Phi$ to decode the embedded vector $\mathbf{z}$ to a real-valued vector~$\hat{\mathbf{y}} \in \mathbb{R}^K$;
the second step is choosing a threshold to transform $\hat{\mathbf{y}}$ to $\tilde{\mathbf{y}} \in \{0, 1\}^{K}$.
If the embedding function $\Phi$ is a linear function, the corresponding $\psi$ can be efficiently computed by pseudo-inverse.
However, for complicated cost functions, a linear function may not be sufficient to properly embed the cost information.
On the other hand, if the embedding function~$\Phi$ is a non-linear function, such as those within kernel principal component analysis \citep{scholkopf1998kpca} and kernel dependency estimation \citep{weston2002kde}, $\psi$ is often difficult to derive or time-consuming in calculation, 
which then makes $\Psi$ practically infeasible to compute.

To resolve the difficulties, we do not consider the two-step decoding function~$\Psi$ that depends on deriving~$\psi$ from~$\Phi$.
Instead, we first fix a decent decoding function~$\Psi$ and then derive the corresponding embedding function~$\Phi$.
We realize that the goal of $\Psi$ is simply to locate the most probable label vector~$\tilde{\mathbf{y}}$ from $\mathcal{Y}$, which is of a finite cardinality, based on the predicted embedded vector $\tilde{\mathbf{z}} = g(\mathbf{x}) \in \mathcal{Z}$.
If all the embedded vectors are sufficiently far away from each other in~$\mathcal{Z}$, one natural decoding function is to calculate the nearest neighbor $\mathbf{z}_q$ of $\tilde{\mathbf{z}}$ and return the corresponding~$\mathbf{y}_q$ as~$\tilde{\mathbf{y}}$. Such a nearest-neighbor decoding function $\Psi$ is behind some ECC-based LSDE algorithms \citep{Ferng2013ecc} and will be adopted. 

The nearest-neighbor decoding function $\Psi$ is based on the distance measure of $\mathcal{Z}$, 
which matches our primary need of representing the cost information.
In particular, if $\mathbf{y}_i$ is a lower-cost prediction than~$\mathbf{y}_j$ with respect to the ground truth $\mathbf{y}_t$, 
we hope that the corresponding embedded vector~$\mathbf{z}_i$ would be closer to~$\mathbf{z}_t$ than $\mathbf{z}_j$. Then, even if $g$ makes a small error such that $\tilde{\mathbf{z}} = g(\mathbf{x})$ deviates from the desired~$\mathbf{z}_t$, nearest-neighbor decoding function $\Psi$
can decode to the lower-cost~$\mathbf{y}_i$ as $\tilde{\mathbf{y}}$ instead of~$\mathbf{y}_j$.
In other words, for any two label vectors $\mathbf{y}_i, \mathbf{y}_j \in \mathcal{Y}$ and the corresponding embedded vectors $\mathbf{z}_i, \mathbf{z}_j \in \mathcal{Z}$, 
we want the Euclidean distance between~$\mathbf{z}_i$ and~$\mathbf{z}_j$, which is denoted by $d( \mathbf{z}_i, \mathbf{z}_j)$, to preserve the magnitude-relationship of the cost~$c(\mathbf{y}_i, \mathbf{y}_j)$.

Based on this objective, the framework of the proposed algorithm is as follows.
In the training stage, for each label vector $\mathbf{y}_i \in \mathcal{Y}$, the proposed algorithm determines an embedded vector $\mathbf{z}_i$ such that
the distance between two embedded vectors $d( \mathbf{z}_i, \mathbf{z}_j)$ in $\mathcal{Z}$ approximates the transformed cost $\delta (c( \mathbf{y}_i, \mathbf{y}_j))$, where $\delta (\cdot)$ is a monotonic transform function to preserve the magnitude-relationship and will be discussed later.
We let the embedding function $\Phi$ be the mapping $\mathbf{y}_i \rightarrow \mathbf{z}_i$ and use $\mathcal{Q}$ to represent the embedded vector set $\{\Phi(\mathbf{y}_i) \: | \: \mathbf{y}_i \in \mathcal{Y} \}$.
Then the algorithm trains a regressor $g\colon \mathcal{X} \rightarrow \mathcal{Z}$ as the internal predictor.

In the predicting stage, when receiving a testing instance~$\mathbf{x}$, 
the algorithm obtains the predicted embedded vector $\mathbf{\tilde{z}} = g(\mathbf{x})$.
Given that the cost information is embedded in the distance, for each $\mathbf{z}_i \in \mathcal{Q}$, the distance $d( \mathbf{z}_i, \mathbf{\tilde{z}})$ can be viewed as the estimated cost if the underlying truth is $\mathbf{y}_i$.
Hence the algorithm finds $\mathbf{z}_q \in \mathcal{Q}$ such that the distance $d( \mathbf{z}_q, \mathbf{\tilde{z}})$ is the smallest (the smallest estimated cost),
and lets the corresponding $\mathbf{y}_q = \Phi^{-1}(\mathbf{z}_q) = \tilde{\mathbf{y}}$ be the final prediction for $\mathbf{x}$.
In other words, we have a nearest-neighbor-based~$\Psi$, with the first step being the determination of the nearest-neighbor of~ $\tilde{\mathbf{z}}$ and the second step being the utilization of $\Phi^{-1}$ to obtain the prediction~$\tilde{\mathbf{y}}$.

Three key issues of the framework above are yet to be addressed.
The first issue is the determination of the embedded vectors~$\mathbf{z}_i$.
The second issue is using the symmetric distance measure to embed the possibly asymmetric cost functions where $c( \mathbf{y}_i, \mathbf{y}_j) \neq c( \mathbf{y}_j, \mathbf{y}_i)$.
The last issue is the choice of a proper monotonic transform function $\delta(\cdot)$. The issues will be discussed in the following sub-sections.

\subsection{Calculating the embedded vectors by multidimensional scaling}

As mentioned above, our objective is to determine embedded vectors $\mathbf{z}_i$ such that the distance $d( \mathbf{z}_i, \mathbf{z}_j)$ approximates the transformed cost $\delta (c( \mathbf{y}_i, \mathbf{y}_j))$. The objective can be formally defined as minimizing the
\textit{embedding error} $$(d(\mathbf{z}_i, \mathbf{z}_j) - \delta (c(\mathbf{y}_i, \mathbf{y}_j)))^2.$$
We observe that the transformed cost $\delta (c( \mathbf{y}_i, \mathbf{y}_j))$ can be viewed as the dissimilarity between label vectors $\mathbf{y}_i$ and $\mathbf{y}_j$.
Computing an embedding based on the dissimilarity information matches the task of manifold learning, which is able to preserve the information and discover the hidden structure. 
Based on our discussions above, any approach that solves the manifold learning task can then be taken to solve the CSLE problem. Nevertheless, for CSLE, different cost functions may need different~$M$ (the dimension of~$\mathcal{Z}$) to achieve a decent embedding. We thus
consider manifold learning approaches that can flexibly take~$M$ as the parameter, and 
adopt a classic manifold learning approach called multidimensional scaling (MDS) \citep{kruskal1964mds}.
%instead of other popular manifold learning approaches, such as locally linear embedding \citep{Roweis2000lle}.

%Now, we are going to resolve the second issue, the determination of the embedded vectors to reduce the \textit{embedding error} $(d(\mathbf{z}_i, \mathbf{z}_j) - \delta (c(\mathbf{y}_i, \mathbf{y}_j)))^2$.
%The transformed cost $\delta (c( \mathbf{y}_i, \mathbf{y}_j))$ can be viewed as a dissimilarity measure between label vectors.
%Computing an embedding based on the dissimilarity information matches the task of manifold learning, which is able to preserve the information and discover the hidden structures. 
%Based on our discussions above, any approach that solves the manifold learning task can then be taken to solve the CSLE problem. Nevertheless, for CSLE, different cost functions may need different~$M$ (the dimension of~$\mathcal{Z}$) to achieve a decent embedding. 
%Thus, it needs to consider manifold learning approaches that can flexibly take~$M$ as the parameter.
%The need motivates us to adopt a classic manifold learning approach called multidimensional scaling (MDS) \citep{kruskal1964mds} instead of other popular manifold learning approaches, such as locally linear embedding \citep{Roweis2000lle}.

For a target dimension~$M$, MDS attempts to discover the hidden structure of~$L_{\scriptscriptstyle MDS}$ objects by embedding their dissimilarities in an $M$-dimensional target space.
The dissimilarity is represented by a symmetric, non-negative, and zero-diagonal dissimilarity matrix $\mathbf{\Delta}$, which is an $L_{\scriptscriptstyle MDS} \times L_{\scriptscriptstyle MDS}$ matrix with $\mathbf{\Delta}_{i, j}$ being the dissimilarity between the $i$-th object and the $j$-th object.
The objective of MDS is to determine target vectors $\mathbf{u}_1, \mathbf{u}_2, ... , \mathbf{u}_{L_{\scriptscriptstyle MDS}}$ in the target space to minimize the \textit{stress}, which is defined as $\sum_{i, j} \mathbf{W}_{i, j}( d(\mathbf{u}_i, \mathbf{u}_j) - \mathbf{\Delta}_{i, j})^2$, where $d$ denotes the Euclidean distance in the target space, and $\mathbf{W}$ is a symmetric, non-negative, and zero-diagonal matrix that carries the weight~$\mathbf{W}_{i, j}$ of each object pair.
There are several algorithms available in the literature for solving MDS.
A representative algorithm is Scaling by MAjorizing a COmplicated Function~(SMACOF) \citep{De1977smacof}, which can iteratively minimize \textit{stress}. The complexity of SMACOF is generally $\mathcal{O}((L_{\scriptscriptstyle MDS})^3)$, but there is often room for speeding up with special weight matrices $\mathbf{W}$.

The \textit{embedding error} $(d(\mathbf{z}_i, \mathbf{z}_j) - \delta (c(\mathbf{y}_i, \mathbf{y}_j)))^2$ and the \textit{stress} $( d(\mathbf{u}_i, \mathbf{u}_j) - \mathbf{\Delta}_{i, j})^2$ are of very similar form.
Therefore, we can view the transformed costs as the dissimilarities of embedded vectors and feed MDS with specific values of $\mathbf{\Delta}$ and $\mathbf{W}$ to
calculate the embedded vectors to reduce the \textit{embedding error}.
Specifically, the relation between MDS and our objective can be described as follows.

\begin{table}[!ht]
	\centering
	\begin{tabular}{|c|c|c|c|c|}
		\hline
		\multirow{2}{*}{$i$-th object} & \multirow{2}{*}{dissimilarity $\mathbf{\Delta}_{i, j}$}   & \multirow{2}{*}{target vector $\mathbf{u}_i$} & \textit{stress}  \\
		                                &  &  & $( d(\mathbf{u}_i, \mathbf{u}_j) - \mathbf{\Delta}_{i, j})^2$ \\
		\hline
		\multirow{2}{*}{label vector $\mathbf{y}_i$} & transformed cost                          & \multirow{2}{*}{embedded vector $\mathbf{z}_i$} & \textit{embedding error} \\
		                                & $\delta (c( \mathbf{y}_i, \mathbf{y}_j))$ &  & $(d(\mathbf{z}_i, \mathbf{z}_j) - \delta (c(\mathbf{y}_i, \mathbf{y}_j)))^2$ \\
		\hline
	
	\end{tabular}
\end{table}

The most complete embedding would convert all label vectors $\mathbf{y} \in \mathcal{Y} \subseteq \{0, 1\}^K$ to the embedded vectors.
Nevertheless, the number of all label vectors is $2^{K}$, which makes solving MDS infeasible.
Therefore, we do not consider embedding the entire~$\mathcal{Y}$.
Instead, we select some representative label vectors as a candidate set $\mathcal{S} \subseteq \mathcal{Y}$, and only embed the label vectors in $\mathcal{S}$.
While the use of $\mathcal{S}$ instead of $\mathcal{Y}$ restricts the nearest-neighbor decoding function to only predict from $\mathcal{S}$, 
it can reduce the computational burden.
One reasonable choice of~$\mathcal{S}$ is the set of label vectors that appear in the training instances~$\mathcal{D}$, which is denoted as $\mathcal{S}_{tr}$.
We will show that using $\mathcal{S}_{tr}$ as~$\mathcal{S}$ readily leads to promising performance and discuss more about the choice of the candidate set in Section~\ref{sec:exp}.

After choosing $\mathcal{S}$, we can construct $\mathbf{\Delta}$ and~$\mathbf{W}$ for solving MDS.
Let $L$ denote the number of elements in $\mathcal{S}$ and let $\mathbf{C}(\mathcal{S})$ be the transformed cost matrix of~$\mathcal{S}$, which is an $L \times L$ matrix with $\mathbf{C}(\mathcal{S})_{i, j} = \delta (c(\mathbf{y}_i, \mathbf{y}_j))$ for $\mathbf{y}_i, \mathbf{y}_j \in \mathcal{S}$.
Unfortunately, $\mathbf{C}(\mathcal{S})$ cannot be directly used as the symmetric dissimilarity matrix $\mathbf{\Delta}$ because the cost function $c(\cdot, \cdot)$ may be asymmetric ($c(\mathbf{y}_i, \mathbf{y}_j) \neq c(\mathbf{y}_j, \mathbf{y}_i)$).
To resolve this difficulty, we propose a \emph{mirroring trick} to construct a symmetric $\mathbf{\Delta}$ from~$\mathbf{C}(\mathcal{S})$. 

\subsection{Mirroring trick for asymmetric cost function}

The asymmetric cost function implies that each label vector $\mathbf{y}_i$ serves two roles: as the ground truth, or as the prediction.
When $\mathbf{y}_i$ serves as the ground truth, we should use $c(\mathbf{y}_i, \cdot)$ to describe the cost behavior.
When $\mathbf{y}_i$ serves as the prediction, we should use $c(\cdot, \mathbf{y}_i)$ to describe the cost behavior.
This motivates us to view these two roles separately.

For each $\mathbf{y}_i \in \mathcal{S}$, we mirror it as~$\mathbf{y}^{(t)}_i$ and~$\mathbf{y}^{(p)}_i$ to denote viewing $\mathbf{y}_i$ as the ground truth and the prediction, respectively. 
Note that the two mirrored label vectors $\mathbf{y}^{(t)}_i$ and $\mathbf{y}^{(p)}_i$ are in fact the same, but carry different meanings.
Now, we have two roles of the candidate sets $\mathcal{S}^{(t)} = \{\mathbf{y}^{(t)}_i\}_{i=1}^{L}$ and $\mathcal{S}^{(p)} = \{\mathbf{y}^{(p)}_i\}_{i=1}^{L}$.
Then, as illustrated by Figure~\ref{fig:embedding}, $\delta (c(\mathbf{y}_i, \mathbf{y}_j))$, the transformed cost when $\mathbf{y}_i$ is ground truth and $\mathbf{y}_j$ is the prediction, can be viewed as the dissimilarity between the ground truth role $\mathbf{y}^{(t)}_i$ and the prediction role~$\mathbf{y}^{(p)}_j$, which is symmetric for them.
Similarly, $\delta (c(\mathbf{y}_j, \mathbf{y}_i))$ can be viewed as the dissimilarity between prediction role~$\mathbf{y}^{(p)}_i$ and ground truth role $\mathbf{y}^{(t)}_j$.
That is, all the asymmetric transformed costs can be viewed as the dissimilarities between the label vectors in $\mathcal{S}^{(t)}$ and $\mathcal{S}^{(p)}$.

\begin{figure}[t]
	\begin{minipage}[b]{.52\columnwidth}
		\begin{tikzpicture}[minimum height=1em, minimum width=1em, line width=0.5pt, draw=cr]
			\node [rec, minimum height=6.5em,  minimum width=17em] at (4.0, 1.0) [] {};
			\node [] at (4.0, -0.3) {embedded space $\mathcal{Z}$};
			\node [z1] at (5.8,0.9) (zp1) {};
			\node [z2] at (4.4,1.7) (zp2) {};
			\node [z3] at (4.8,0.4) (zp3) {};
			\node [] at (6.1,0.6) {{\color{black!70} $\mathbf{z}^{(p)}_1$}};
			\node [] at (4.9,1.7) {{\color{black!70} $\mathbf{z}^{(p)}_2$}};
			\node [] at (5.3,0.4) {{\color{black!70} $\mathbf{z}^{(p)}_3$}};
			\node [z1, fill=none, draw=c1] at (3.1,0.6) (zt1) {};
			\node [z2, fill=none, draw=c2] at (2.6,1.7) (zt2) {};
			\node [z3, fill=none, draw=c3] at (2.4,0.3) (zt3) {};
			\node [] at (3.3,0.3) {{\color{black!70} $\mathbf{z}^{(t)}_1$}};
			\node [] at (2.1,1.7) {{\color{black!70} $\mathbf{z}^{(t)}_2$}};
			\node [] at (1.95,0.3) {{\color{black!70} $\mathbf{z}^{(t)}_3$}};
			\draw [draw=red!30!yellow] (zt1) edge node {} (zp2);
			\draw [draw=red!30!yellow] (zt2) edge node {} (zp1);
			\node [text=blue!60!red] at (2.8,1.0) {$\delta (c(\mathbf{y}_1, \mathbf{y}_2))$};
			\node [text=blue!60!red] at (5.0,1.3) {$\delta (c(\mathbf{y}_2, \mathbf{y}_1))$};
		\end{tikzpicture}
		\caption{Embedding cost in distance}
		\label{fig:embedding}
	\end{minipage}
	\begin{minipage}[b]{.46\columnwidth}
		\subfigure[$\mathbf{\Delta}$]{\includegraphics[width=.40\columnwidth]{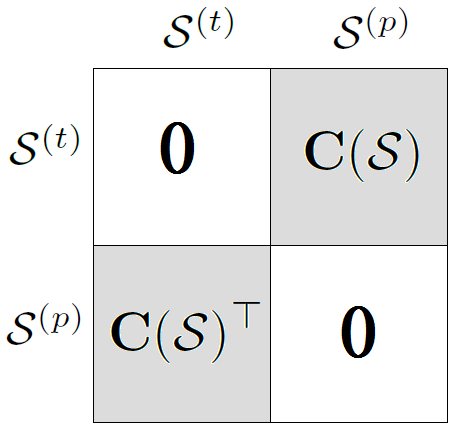}}
		\hspace*{0.5em}
		\subfigure[$\mathbf{W}$]{\includegraphics[width=.40\columnwidth]{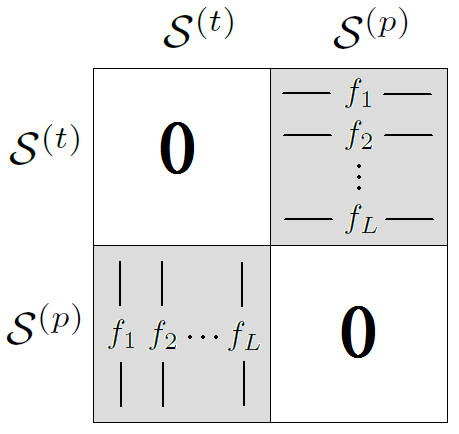}}
		\caption{Constructions of $\mathbf{\Delta}$ and $\mathbf{W}$}
		\label{fig:mat}
	\end{minipage}

\end{figure}

Based on this view, instead of embedding $\mathcal{S}$ by MDS, we embed both $\mathcal{S}^{(t)}$ and~$\mathcal{S}^{(p)}$ by considering
$2 L$ objects, the first~$L$ objects being the elements in~$\mathcal{S}^{(t)}$ and the last~$L$ objects being the elements in $\mathcal{S}^{(p)}$.
Following the mirroring step above, we construct symmetric $\mathbf{\Delta}$ and~$\mathbf{W}$ as $2 L \times 2 L$ matrices by the following equations and illustrate the constructions by Figure~\ref{fig:mat}.
\begin{equation}
\label{eq:delta}
\mathbf{\Delta}_{i,j} = \begin{cases}
	\delta (c(\mathbf{y}_i, \mathbf{y}_{j-L})) & \text{if $(i,j)$ in top-right part} \\
	\delta (c(\mathbf{y}_{i-L}, \mathbf{y}_j)) & \text{if $(i,j)$ in bottom-left part}  \\
	0 & \text{otherwise} \\
	\end{cases}
\end{equation}
\begin{equation}
\label{eq:weight}
\mathbf{W}_{i,j} = \begin{cases}
	f_i & \text{if $(i,j)$ in top-right part} \\
	f_j & \text{if $(i,j)$ in bottom-left part} \\
	0 & \text{otherwise} \\
	\end{cases}
\end{equation}
We explain the constructions and the new notations $f_i$ as follows. Given that we are concerned only about the dissimilarities between the elements in $\mathcal{S}^{(t)}$ and $\mathcal{S}^{(p)}$, we set the top-left and the bottom-right parts of $\mathbf{W}$ to zeros (and set the corresponding parts of $\mathbf{\Delta}$ conveniently to zeros as well). Then,
we set the top-right part and the bottom-left part of $\mathbf{\Delta}$ to be the transformed costs to reflect the dissimilarities.
The top-right part and the bottom-left part of $\mathbf{\Delta}$ are in fact~$\mathbf{C}(\mathcal{S})$ and $\mathbf{C}(\mathcal{S})^\top$ respectively, as illustrated by Figure~\ref{fig:mat}.
%To echo Theorem~\ref{thm:bound}, which bounds the cost for per instance, we set the top-right part of weight $\mathbf{W}_{i,j}$ to be $f_i$, the frequency of $\mathbf{y}_i$ in $\mathcal{D}$, and set the bottom-left part of weight $\mathbf{W}_{i,j}$ to be $f_j$.
Considering that every label vector may have different importance, to reflect this difference, we set the top-right part of weight $\mathbf{W}_{i,j}$ to be $f_i$, the frequency of $\mathbf{y}_i$ in $\mathcal{D}$, and set the bottom-left part of weight $\mathbf{W}_{i,j}$ to be $f_j$.

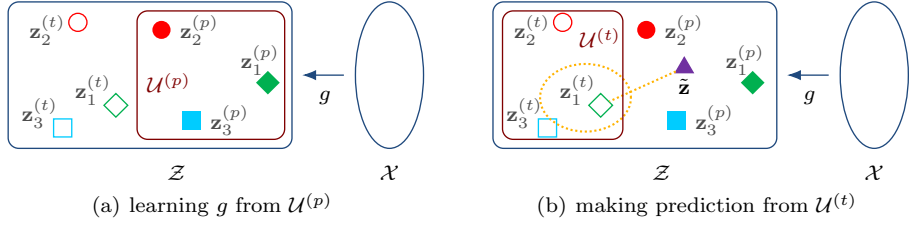
\begin{figure}[t]
	\centering
	\subfigure[learning $g$ from $\mathcal{U}^{(p)}$\label{fig:flow_train}]{
		\begin{tikzpicture}[minimum height=1em, minimum width=1em, line width=0.5pt, draw=cr]
			\node [rec, minimum height=6.5em,  minimum width=12.5em] at (3.35, 1.0) [] {};
			\node [] at (3.7, -0.3) {$\mathcal{Z}$};
			\node [z1] at (4.9,0.9) (zp1) {};
			\node [z2] at (3.5,1.6) (zp2) {};
			\node [z3] at (3.9,0.4) (zp3) {};
			\node [] at (4.8,1.2) {{\color{black!70} $\mathbf{z}^{(p)}_1$}};
			\node [] at (4.0,1.6) {{\color{black!70} $\mathbf{z}^{(p)}_2$}};
			\node [] at (4.4,0.4) {{\color{black!70} $\mathbf{z}^{(p)}_3$}};
			\node [z1, fill=none, draw=c1] at (2.9,0.6) (zt1) {};
			\node [z2, fill=none, draw=c2] at (2.4,1.7) (zt2) {};
			\node [z3, fill=none, draw=c3] at (2.2,0.3) (zt3) {};
			\node [] at (2.6,0.8) {{\color{black!70} $\mathbf{z}^{(t)}_1$}};
			\node [] at (2.0,1.6) {{\color{black!70} $\mathbf{z}^{(t)}_2$}};
			\node [] at (1.9,0.5) {{\color{black!70} $\mathbf{z}^{(t)}_3$}};
			\draw [-Latex, draw=cr] (5.9,1.0) to (5.35,1.0);
			\node [] at (5.65,0.7) {$g$};
			\node [] at (6.5,-0.3) {$\mathcal{X}$};
			\draw [] (6.5, 1.0) ellipse (1.5em and 3.25em);
			\node [] at (3.6,0.9) {{\color{red!50!black} $\mathcal{U}^{(p)}$}};
			\node [rec, minimum height=5.7em,  minimum width=6.5em, draw=red!50!black] at (4.15,1.0) [] {};
		\end{tikzpicture}}
	\hspace{2.0em}
	\subfigure[making prediction from $\mathcal{U}^{(t)}$\label{fig:flow_predict}]{
		\begin{tikzpicture}[minimum height=1em, minimum width=1em, line width=0.5pt, draw=cr]
			\node [rec, minimum height=6.5em,  minimum width=12.5em] at (3.35, 1.0) [] {};
			\node [] at (3.7, -0.3) {$\mathcal{Z}$};
			\node [z1] at (4.9,0.9) (zp1) {};
			\node [z2] at (3.5,1.6) (zp2) {};
			\node [z3] at (3.9,0.4) (zp3) {};
			\node [] at (4.8,1.2) {{\color{black!70} $\mathbf{z}^{(p)}_1$}};
			\node [] at (4.0,1.6) {{\color{black!70} $\mathbf{z}^{(p)}_2$}};
			\node [] at (4.4,0.4) {{\color{black!70} $\mathbf{z}^{(p)}_3$}};
			\node [z1, fill=none, draw=c1] at (2.9,0.6) (zt1) {};
			\node [z2, fill=none, draw=c2] at (2.4,1.7) (zt2) {};
			\node [z3, fill=none, draw=c3] at (2.2,0.3) (zt3) {};
			\node [] at (2.6,0.8) {{\color{black!70} $\mathbf{z}^{(t)}_1$}};
			\node [] at (2.0,1.6) {{\color{black!70} $\mathbf{z}^{(t)}_2$}};
			\node [] at (1.9,0.5) {{\color{black!70} $\mathbf{z}^{(t)}_3$}};
			\draw [-Latex, draw=cr] (5.9,1.0) to (5.35,1.0);
			\node [] at (5.65,0.7) {$g$};
			\node [] at (6.5,-0.3) {$\mathcal{X}$};
			\draw [] (6.5, 1.0) ellipse (1.5em and 3.25em);
			\node [zt] at (4.0,1.1) (zt) {};
			\node [] at (4.0,0.85) {$\tilde{\mathbf{z}}$};
			\draw [draw=red!30!yellow, densely dotted, line width=0.8pt] (zt) edge node {} (zt1);
			\draw [draw=red!30!yellow, densely dotted, line width=0.8pt] (2.7, 0.7) ellipse (2.0em and 1.5em);
			\node [] at (2.9,1.5) {{\color{red!50!black} $\mathcal{U}^{(t)}$}};
			\node [rec, minimum height=5.7em,  minimum width=5.3em, draw=red!50!black] at (2.4,1.0) [] {};
		\end{tikzpicture}}
	\caption{Different use of two roles of embedded vectors}
\end{figure}

By solving MDS with the above-mentioned $\mathbf{\Delta}$ and $\mathbf{W}$, we can obtain the target vector $\mathbf{u}^{(t)}_i$ and $\mathbf{u}^{(p)}_i$ corresponding to $\mathbf{y}^{(t)}_i$ and $\mathbf{y}^{(p)}_i$.
We take $\mathcal{U}^{(t)}$ and $\mathcal{U}^{(p)}$ to denote the target vector sets $\{\mathbf{u}^{(t)}_i\}_{i=1}^{L}$ and $\{\mathbf{u}^{(p)}_i\}_{i=1}^{L}$, respectively.
Those target vectors minimize $\sum_{i, j} \mathbf{W}_{i,j}(d(\mathbf{u}^{(t)}_i, \mathbf{u}^{(p)}_j) - \delta (c(\mathbf{y}_i, \mathbf{y}_j)) )^2$.
That is, the cost information is embedded in the distances between the elements in $\mathcal{U}^{(t)}$ and $\mathcal{U}^{(p)}$.

Since we mirror each label vector $\mathbf{y}_i$ as two roles ($\mathbf{y}^{(t)}_i$ and $\mathbf{y}^{(p)}_i$), we need to decide which target vector ($\mathbf{u}^{(t)}_i$ and $\mathbf{u}^{(p)}_i$) is the embedded vector $\mathbf{z}_i$ of $\mathbf{y}_i$.
Recall that the goal of the embedded vectors is to train a internal predictor $g$ and obtain~$\tilde{\mathbf{z}}$, the ``predicted'' embedded vector.
Therefore, we take the elements in $\mathcal{U}^{(p)}$, which serve the role of the prediction, as the embedded vectors of the elements in $\mathcal{S}$, as illustrated by Figure~\ref{fig:flow_train}.
Accordingly, the nearest embedded vector $\mathbf{z}_q$ should be the role of the ground truth because the cost information is embedded in the distance between these two roles of target vectors.
Hence, we take~$\mathcal{U}^{(t)}$ as $\mathcal{Q}$, the embedded vector set in the first step of nearest-neighbor decoding, and find the nearest embedded vector $\mathbf{z}_q$ from $\mathcal{Q}$, as illustrated by Figure~\ref{fig:flow_predict}.
The final cost-sensitive prediction $\tilde{\mathbf{y}} = \mathbf{y}_q$ is the corresponding label vector to $\mathbf{z}_q$, which carries the cost information through nearest-neighbor decoding.

\begin{algorithm}[t]
\caption{Training process of CLEMS}
\label{alg:train}
\begin{algorithmic}[1]
	\STATE Given $\mathcal{D} = \{(\mathbf{x}^{(n)}, \mathbf{y}^{(n)}) \}_{n=1}^{N}$, cost function $c$, and embedded dimension $M$
	\STATE Decide the candidate set $\mathcal{S}$, and calculate $\mathbf{\Delta}$ and $\mathbf{W}$ by (\ref{eq:delta}) and (\ref{eq:weight})
	\STATE Solve MDS with $\mathbf{\Delta}$ and $\mathbf{W}$, 
		 and obtain the two roles of embedding vectors $\mathcal{U}^{(t)}$ and $\mathcal{U}^{(p)}$
	\STATE Set embedding function $\Phi\colon \mathcal{S} \rightarrow \mathcal{U}^{(p)}$ and embedded vector set $\mathcal{Q} = \mathcal{U}^{(t)}$
	\STATE Train a regressor $g$ from $\{(\mathbf{x}^{(n)}, \Phi(\mathbf{y}^{(n)})) \}_{n=1}^{N}$
\end{algorithmic}
\end{algorithm}

\begin{algorithm}[t]
\caption{Predicting process of CLEMS}
\label{alg:predict}
\begin{algorithmic}[1]
	\STATE Given a testing example $\mathbf{x}$
	\STATE Obtain the predicted embedded vector $\tilde{\mathbf{z}}=g(\mathbf{x})$
	\STATE Find $\mathbf{z}_q \in \mathcal{Q}$ such that $d( \mathbf{z}_q, \mathbf{\tilde{z}})$ is the smallest
	\STATE Make prediction $\tilde{\mathbf{y}}=\Phi^{-1}(\mathbf{z}_q)$
\end{algorithmic}
\end{algorithm}

With the embedding function $\Phi$ using $\mathcal{U}^{(p)}$ and the nearest-neighbor decoding function $\Psi$ using $\mathcal{Q} = \mathcal{U}^{(t)}$, we have now designed a novel CSLE algorithm. We name it cost-sensitive label embedding with multidimensional scaling (CLEMS).
Algorithm \ref{alg:train} and Algorithm \ref{alg:predict} respectively list the training process and the predicting process of CLEMS.

\subsection{Theoretical guarantee and monotonic function}
The last issue is how to choose the monotonic transform function $\delta(\cdot)$.
We suggest a proper monotonic function $\delta(\cdot)$ based on the following theoretical results.

\begin{theorem}
\label{thm:bound}
For any instance $(\mathbf{x}, \mathbf{y})$, let $\mathbf{z}$ be the embedded vector of $\mathbf{y}$,
$\tilde{\mathbf{z}} = g(\mathbf{x})$ be the predicted embedded vector, 
$\mathbf{z}_q$ be the nearest embedded vector of $\tilde{\mathbf{z}}$,
and $\mathbf{y}_q$ be the corresponding label vector of $\mathbf{z}_q$. In other words, $\mathbf{y}_q$ is the outcome of the nearest-neighbor decoding function~$\Psi$.
Then,
\[\delta (c(\mathbf{y}, \mathbf{y}_q))^2 \leq 5 \Bigl(\underbrace{(d(\mathbf{z}, \mathbf{z}_q) - \delta (c(\mathbf{y}, \mathbf{y}_q)))^2}_{\mbox{embedding error}} + \underbrace{d(\mathbf{z}, \tilde{\mathbf{z}})^2}_{\mbox{regression error}}\Bigr).\]
\end{theorem}

\begin{proof}
Since $\mathbf{z}_q$ is the nearest neighbor of $\tilde{\mathbf{z}}$,
we have $d(\mathbf{z}, \tilde{\mathbf{z}}) \geq \frac{1}{2} d(\mathbf{z}, \mathbf{z}_q)$.
Hence, 
\begin{align*}
\textit{embedding error} + \textit{regression error} &= (d(\mathbf{z}, \mathbf{z}_q) - \delta (c(\mathbf{y}, \mathbf{y}_q)))^2 + d(\mathbf{z}, \tilde{\mathbf{z}})^2 \\
&\geq (d(\mathbf{z}, \mathbf{z}_q) - \delta (c(\mathbf{y}, \mathbf{y}_q)))^2 + \frac{1}{4} d(\mathbf{z}, \mathbf{z}_q)^2 \\
&= \frac{5}{4} (d(\mathbf{z}, \mathbf{z}_q) - \frac{4}{5} \delta (c(\mathbf{y}, \mathbf{y}_q)))^2 + \frac{1}{5} \delta (c(\mathbf{y}, \mathbf{y}_q))^2 \\
&\geq \frac{1}{5} \delta (c(\mathbf{y}, \mathbf{y}_q))^2 .
\end{align*}
This implies the theorem.
\end{proof}

Theorem~\ref{thm:bound} implies that the cost of the prediction can be bounded by \textit{embedding error} and \textit{regression error}.
In our framework, the \textit{embedding error} can be reduced by multidimensional scaling and the \textit{regression error} can be reduced by learning a good regressor~$g$.
Theorem~\ref{thm:bound} provides a theoretical explanation of how our framework achieves cost-sensitivity.

In general, any monotonic function $\delta(\cdot)$ can be used in the proposed framework.
Based on Theorem~\ref{thm:bound}, we suggest $\delta (\cdot) = (\cdot)^{1/2}$ to directly bound the cost by $c(\mathbf{y}, \mathbf{y}_q) \leq 5 (\textit{embedding error} + \textit{regression error})$.
We will show that the suggested monotonic function leads to promising practical performance in Section~\ref{sec:exp}.

%%%%%%%%%%%%%%%%%%%%%%%%%%%%%%%%%%%%%%%% Experiments %%%%%%%%%%%%%%%%%%%%%%%%%%%%%%%%%%%%%%%%
\section{Experiments}
\label{sec:exp}
We conduct the experiments on nine real-world datasets \citep{Tsoumakas2011mulan,Read2016meka} to validate the proposed algorithm, CLEMS.
The details of the datasets are shown by Table~\ref{tab:data}.
We evaluate the algorithms in our cost-sensitive setting with three commonly-used evaluation criteria, namely
\textit{F1 score}$(\mathbf{y}, \tilde{\mathbf{y}}) = \frac{2 \|\mathbf{y} \cap \tilde{\mathbf{y}} \|_1 }{\|\mathbf{y} \|_1 + \| \tilde{\mathbf{y}} \|_1}$,
\textit{Accuracy score}$(\mathbf{y}, \tilde{\mathbf{y}}) = \frac{\|\mathbf{y} \cap \tilde{\mathbf{y}} \|_1 }{\|\mathbf{y} \cup \tilde{\mathbf{y}} \|_1}$,
and \textit{Rank loss}$(\mathbf{y}, \tilde{\mathbf{y}}) = \sum\limits_{\mathbf{y}[i]>\mathbf{y}[j]} ( \llbracket \tilde{\mathbf{y}}[i] < \tilde{\mathbf{y}}[j] \rrbracket + \frac{1}{2}\llbracket \tilde{\mathbf{y}}[i] = \tilde{\mathbf{y}}[j] \rrbracket )$.
Note that F1 score and Accuracy score are symmetric while Rank loss is asymmetric.
For CLEMS, the input cost function is set as the corresponding evaluation criterion.

All the following experimental results are averaged over $20$ runs of experiments.
In each run, we randomly split $50\%$, $25\%$, and $25\%$ of the dataset for training, validation, and testing.
We use the validation part to select the best parameters for all the algorithms and report the corresponding testing results.
For all the algorithms, the internal predictors are set as random forest \citep{Breiman2001rfr} implemented by scikit-learn \citep{Pedregosa2011sklearn} and the maximum depth of the trees is selected from $\{ 5, 10, ..., 35 \}$.
For CLEMS, we use the implementation of scikit-learn for solving SMACOF algorithm to obtain the MDS-based embedding and the parameters of SMACOF algorithm are set as default values by scikit-learn.
For the other algorithms, the rest parameters are set as the default values suggested by their original papers.
In the following figures and tables, we use the notation $\uparrow (\downarrow)$ to highlight whether a higher (lower) value indicates better performance for the evaluation criterion.

\begin{table}[t]
	\caption{Properties of datasets}
	\label{tab:data}
	\scriptsize
	\centering
	\begin{tabular}{ccccc}
		\hline
		Dataset 	& \# of instance $N$ & \# of feature $d$ & \# of labels $K$  & \# of distinct labels \\
		\hline
		CAL500    	& 502 	& 68  	& 174  	& 502 	\\
		emotions 	& 593  	& 72  	& 6   	& 27  	\\
		birds    	& 645  	& 260 	& 19  	& 133 	\\
		medical  	& 978   	& 1449 	& 45  	& 94  	\\
		enron     	& 1702 	& 1001 	& 53  	& 753 	\\
		scene    	& 2407 	& 294 	& 6   	& 15  	\\
		yeast    	& 2417 	& 103 	& 14  	& 198 	\\
		slashdot 	& 3279  	& 1079 	& 22 	& 156	\\
		EUR-Lex(dc)	& 19348 	& 5000 	& 412 	& 1615	\\
		\hline
	\end{tabular}
\end{table}

%%%%%%%%%%%%%%%%%%%%%%%%%%%%%%%%%%%%%%%% LSDR %%%%%%%%%%%%%%%%%%%%%%%%%%%%%%%%%%%%%%%%
\subsection{Comparing CLEMS with LSDR algorithms}
In the first experiment, we compare CLEMS with four LSDR algorithms introduced in Section~\ref{sec:csle}:
principal label space transformation (PLST) \citep{Tai2012plst}, 
conditional principal label space transformation (CPLST) \citep{Chen2012cplst},
feature-aware implicit label space encoding (FaIE) \citep{Lin2014faie},
and sparse local embeddings for extreme classification (SLEEC) \citep{Bhatia2015sleec}

Since the prediction of SLEEC is a real-value vector rather than binary,
we choose the best threshold for quantizing the vector according to the given criterion during training.
Thus, our modified SLEEC can be viewed as ``semi-cost-sensitive'' algorithm that learns the threshold according to the criterion.

\begin{figure}[!t]
	\centering
	\includegraphics[width=.95\textwidth]{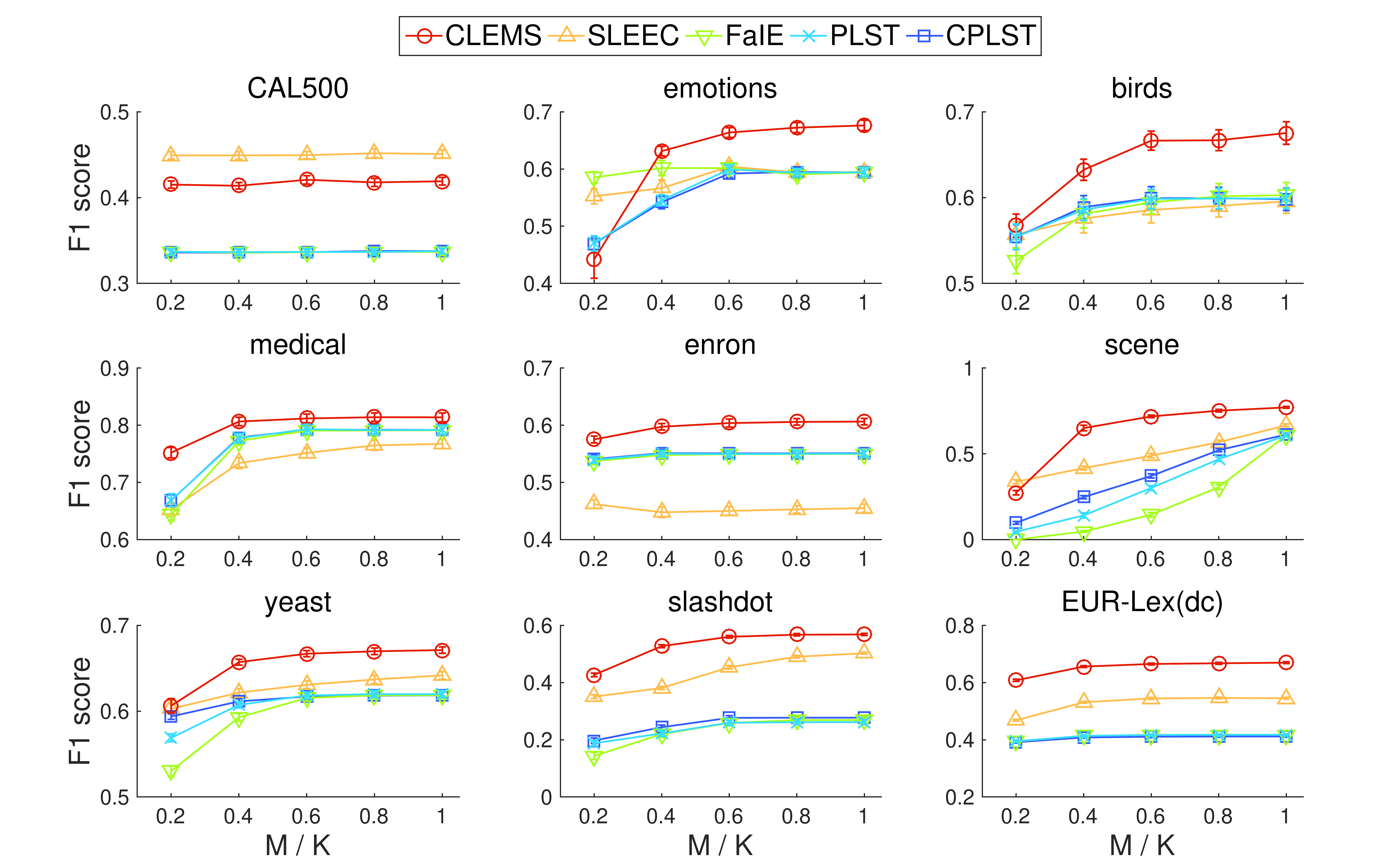}
	\caption{F1 score ($\uparrow$) with the 95\% confidence interval of CLEMS and LSDR algorithms}
	\label{fig:lsdr_f1}
\end{figure}

\begin{figure}[!t]
	\centering
	\includegraphics[width=.95\textwidth]{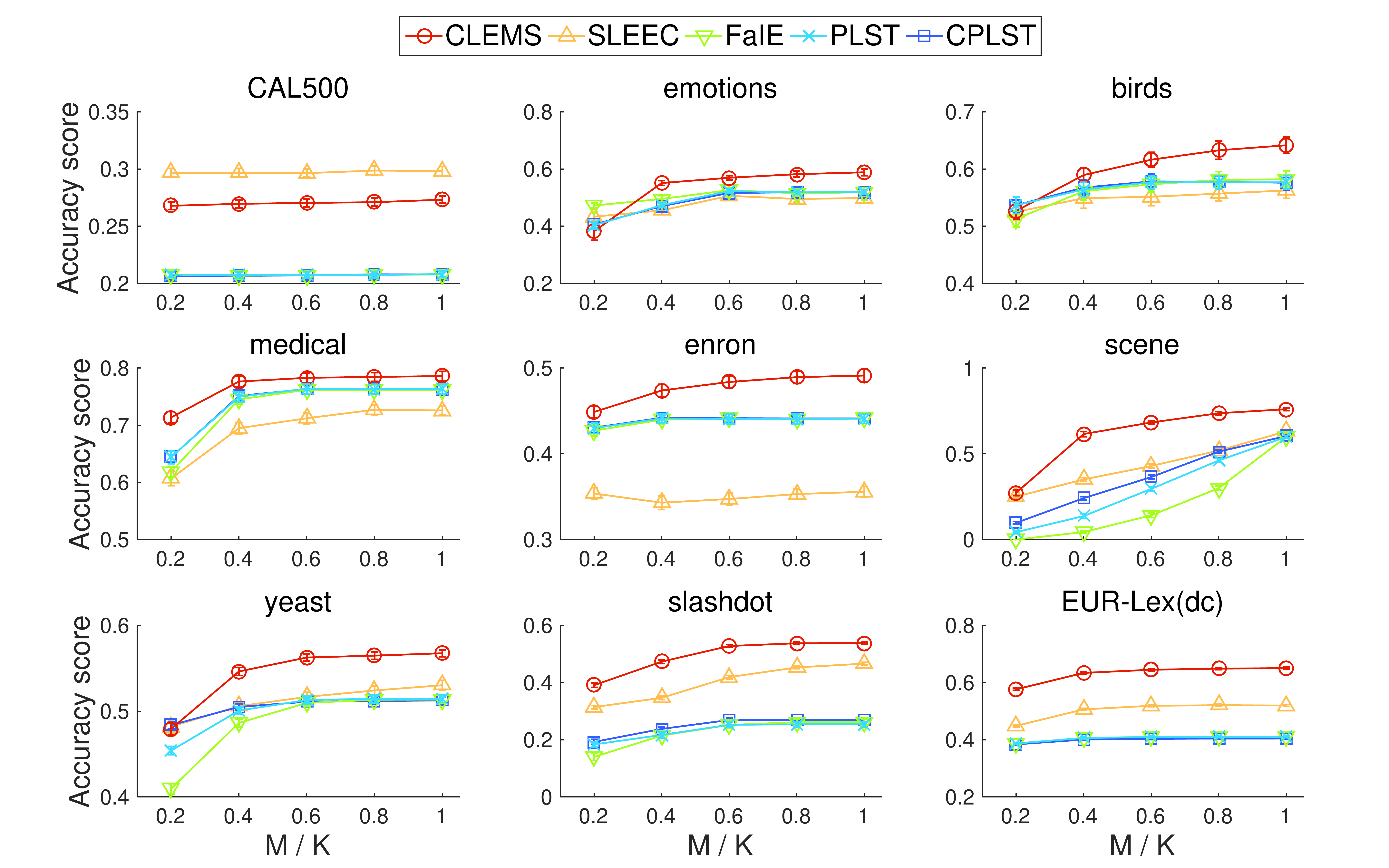}
	\caption{Accuracy score ($\uparrow$) with the 95\% confidence interval of CLEMS and LSDR algorithms}
	\label{fig:lsdr_acc}
\end{figure}

\begin{figure}[!t]
	\centering
	\includegraphics[width=.95\textwidth]{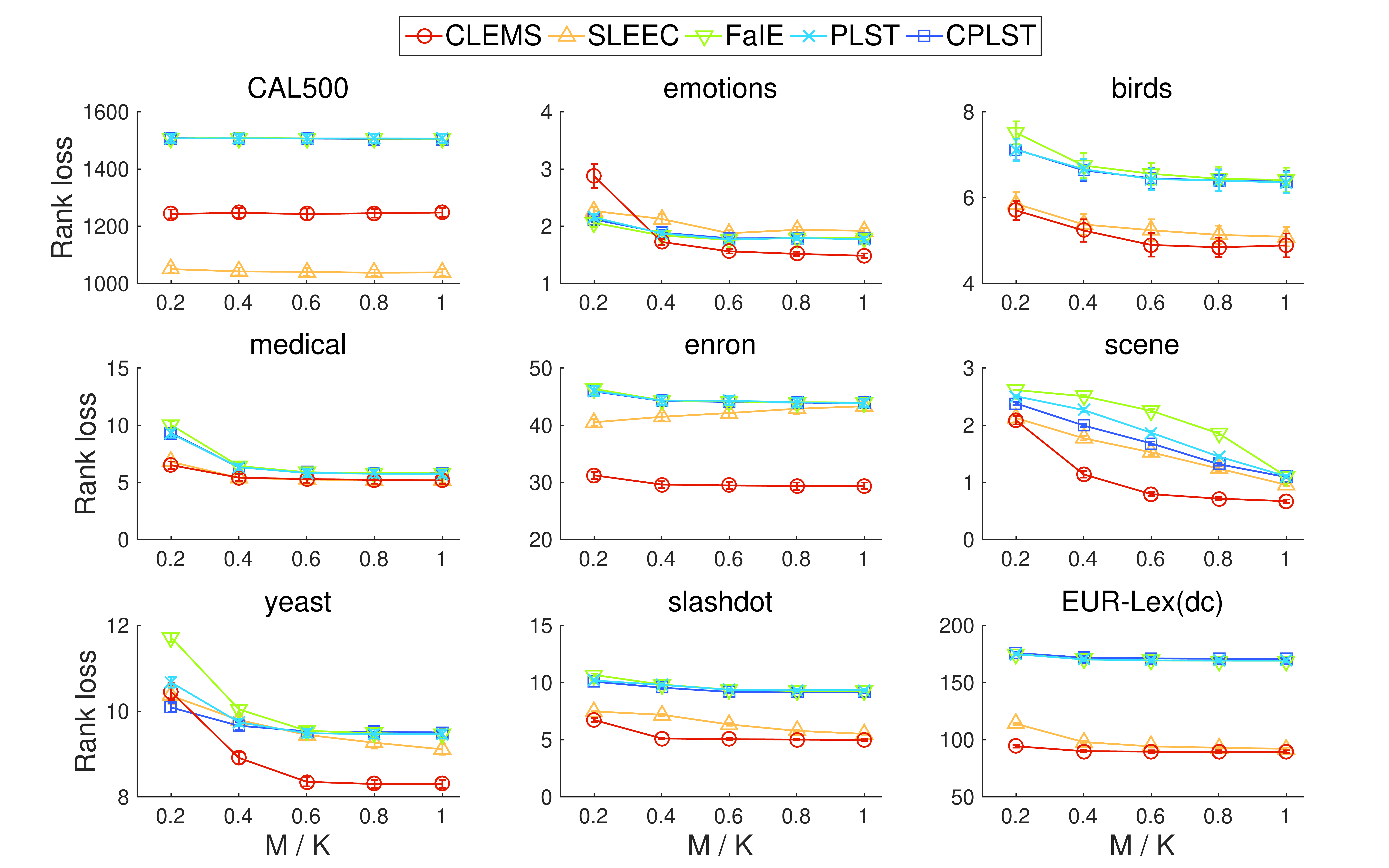}
	\caption{Rank loss ($\downarrow$) with the 95\% confidence interval of CLEMS and LSDR algorithms}
	\label{fig:lsdr_rank}
\end{figure}

Figures~\ref{fig:lsdr_f1} and Figure~\ref{fig:lsdr_acc} show the results of F1 score and Accuracy score across different embedded dimensions~$M$.
As~$M$ increases, all the algorithms reach better performance because of the better preservation of label information.
CLEMS outperforms the non-cost-sensitive algorithms (PLST, CPLST, and FaIE) in most of the cases, which verifies the importance of constructing a cost-sensitive embedding. CLEMS also exhibits considerably better performance over SLEEC in most of the datasets, which demonstrates the usefulness to consider the cost information during embedding (CLEMS) rather than after the embedding (SLEEC).
The results of Rank loss are shown by Figure~\ref{fig:lsdr_rank}.
CLEMS again reaches the best in most of the cases, which justifies its validity for asymmetric criteria through the mirroring trick.

%%%%%%%%%%%%%%%%%%%%%%%%%%%%%%%%%%%%%%%% LSDE %%%%%%%%%%%%%%%%%%%%%%%%%%%%%%%%%%%%%%%%
\subsection{Comparing CLEMS with LSDE algorithms}
We compare CLEMS with ECC-based LSDE algorithms \citep{Ferng2013ecc}.
We consider two promising error-correcting codes, \textit{repetition code} (ECC-RREP) and \textit{Hamming on repetition code} (ECC-HAMR) in the original work.
The former is equivalent to the famous Random $k$-labelsets~(RA$k$EL) algorithm \citep{Tsoumakas2011rakel}.

\begin{figure}[!t]
	\centering
	\includegraphics[width=.95\textwidth]{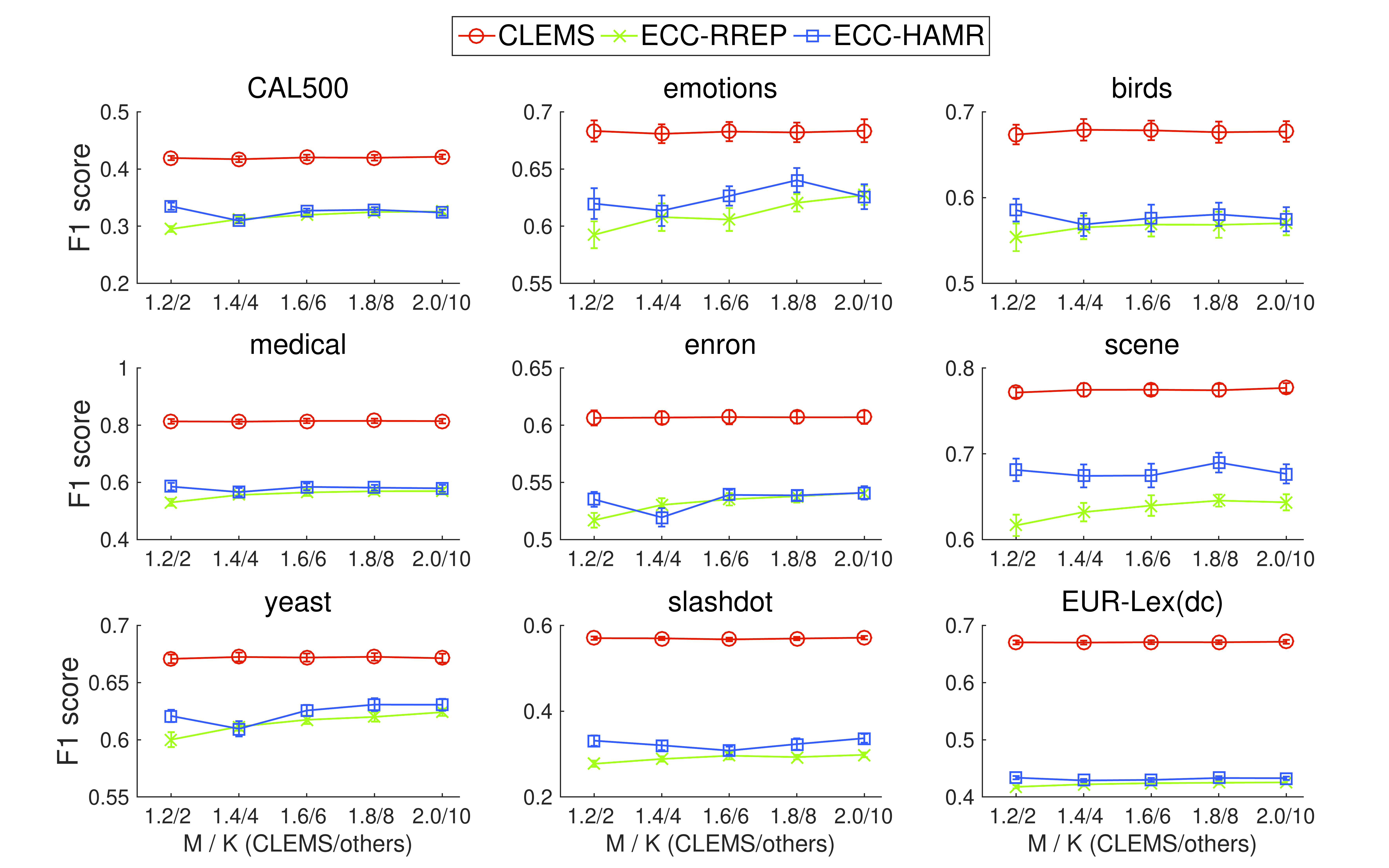}
	\caption{F1 score ($\uparrow$) with the 95\% confidence interval of CLEMS and LSDE algorithms}
	\label{fig:lsde_f1}
\end{figure}

\begin{figure}[!t]
	\centering
	\includegraphics[width=.95\textwidth]{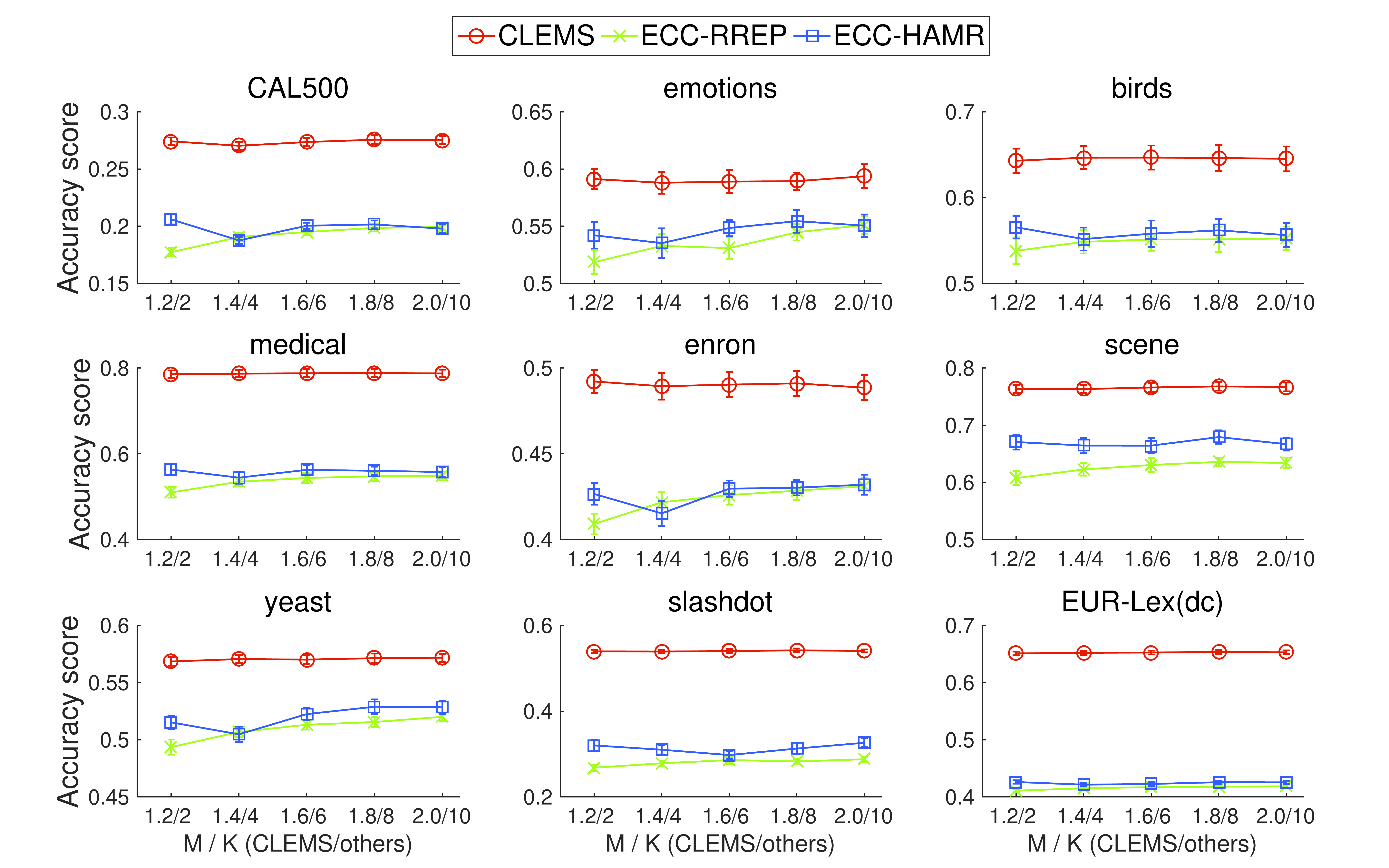}
	\caption{Accuracy score ($\uparrow$) with the 95\% confidence interval of CLEMS and LSDE algorithms}
	\label{fig:lsde_acc}
\end{figure}

\begin{figure}[!t]
	\centering
	\includegraphics[width=.95\textwidth]{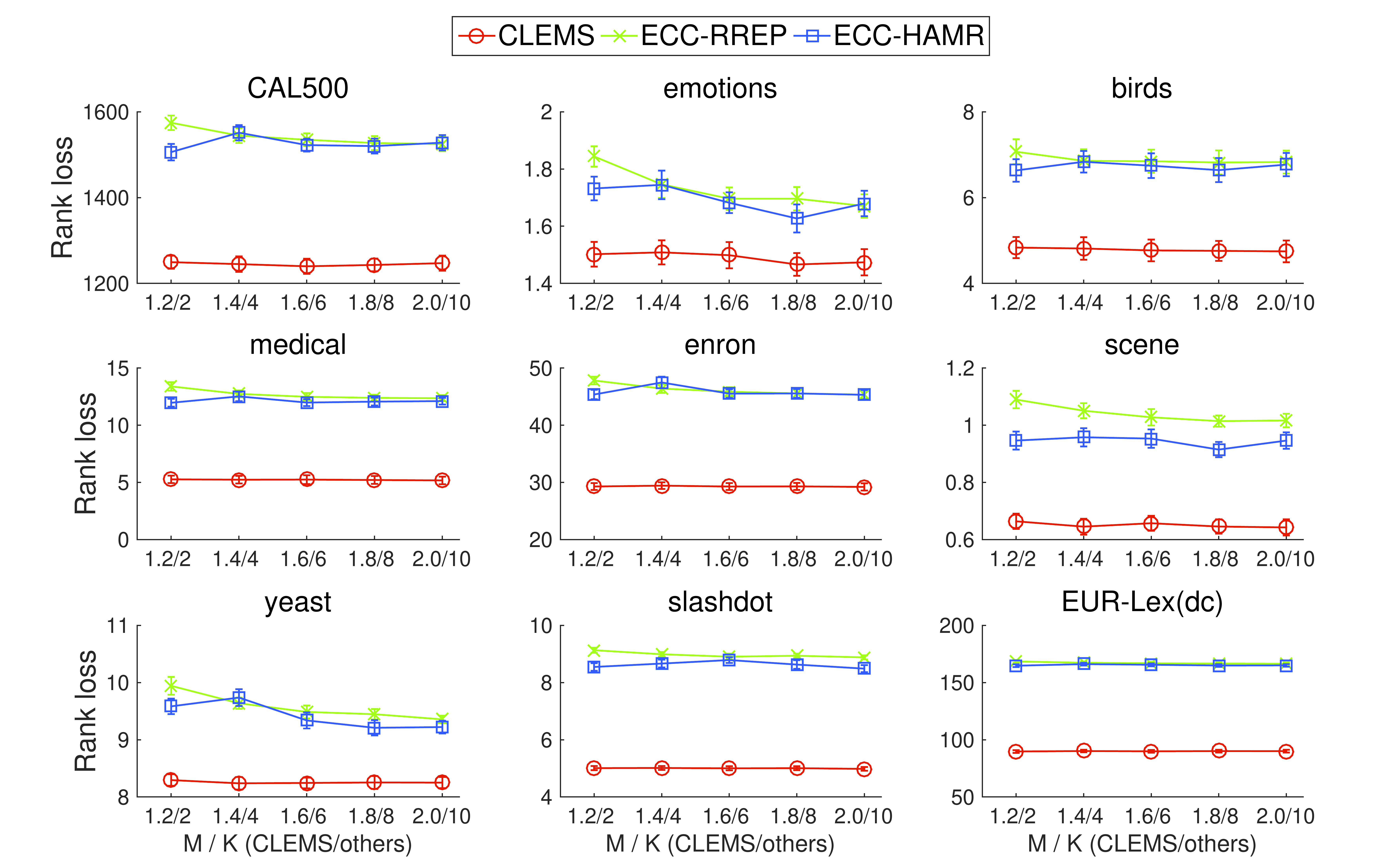}
	\caption{Rank loss ($\downarrow$) with the 95\% confidence interval of CLEMS and LSDE algorithms}
	\label{fig:lsde_rank}
\end{figure}

Figure~\ref{fig:lsde_f1} shows the results of F1 score.
Note that in the figure, the scales of $M / K$ for CLEMS and other LSDE algorithms are different.
The scale of CLEMS is $\{1.2, 1.4, 1.6, 1.8, 2.0\}$ while the scale of other LSDE algorithms is $\{2, 4, 6, 8, 10 \}$.
Although we give LSDE algorithms more dimensions to embed the label information, CLEMS is still superior to those LSDE algorithms in most of cases.
Similar results happen for Accuracy score and the Rank loss (Figure~\ref{fig:lsde_acc} and Figure~\ref{fig:lsde_rank}).
The results again justify the superiority of CLEMS.

%%%%%%%%%%%%%%%%%%%%%%%%%%%%%%%%%%%%%%%% Candidate Set %%%%%%%%%%%%%%%%%%%%%%%%%%%%%%%%%%%%%%%%
\subsection{Candidate set and embedded dimension}
Now, we discuss the influence of the candidate set $\mathcal{S}$.
In Section~\ref{sec:clems}, we proposed to embed $\mathcal{S}_{tr}$ instead of $\mathcal{Y}$.
To verify the goodness of the choice, we compare CLEMS with different candidate sets.
We consider the sets sub-sampled with different percentage from $\mathcal{S}_{tr}$ to evaluate the importance of label vectors in $\mathcal{S}_{tr}$.
Furthermore, to know whether or not larger candidate set leads to better performance, we also randomly sample different percentage of additional label vectors from $\mathcal{Y} \setminus \mathcal{S}_{tr}$ and merge them with $\mathcal{S}_{tr}$ as the candidate sets.
The results of three largest datasets are shown by Figures~\ref{fig:ct_f1}, \ref{fig:ct_acc}, and~\ref{fig:ct_rank}.
From the figures, we observe that sub-sampling from $\mathcal{S}_{tr}$ generally lead to worse performance; adding more candidates from $\mathcal{Y} \setminus \mathcal{S}_{tr}$, on the other hand, does not lead to significantly-better performance.
The two findings suggest that using $\mathcal{S}_{tr}$ as the candidate set is necessary and sufficient for decent performance.

\begin{figure}[!t]
	\centering
	\includegraphics[width=.95\textwidth]{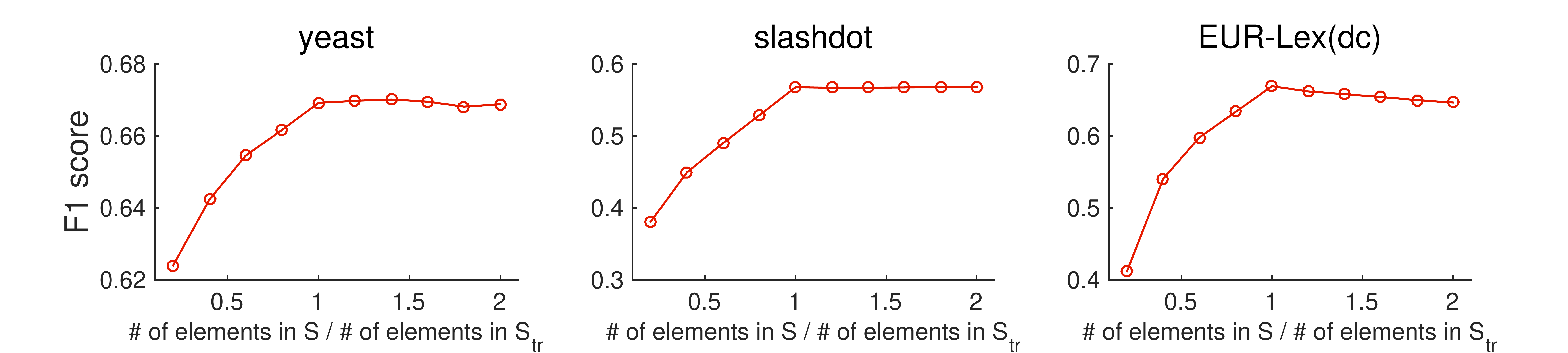}
	\caption{F1 score ($\uparrow$) of CLEMS with different size of candidate sets}
	\label{fig:ct_f1}
\end{figure}

\begin{figure}[!t]
	\centering
	\includegraphics[width=.95\textwidth]{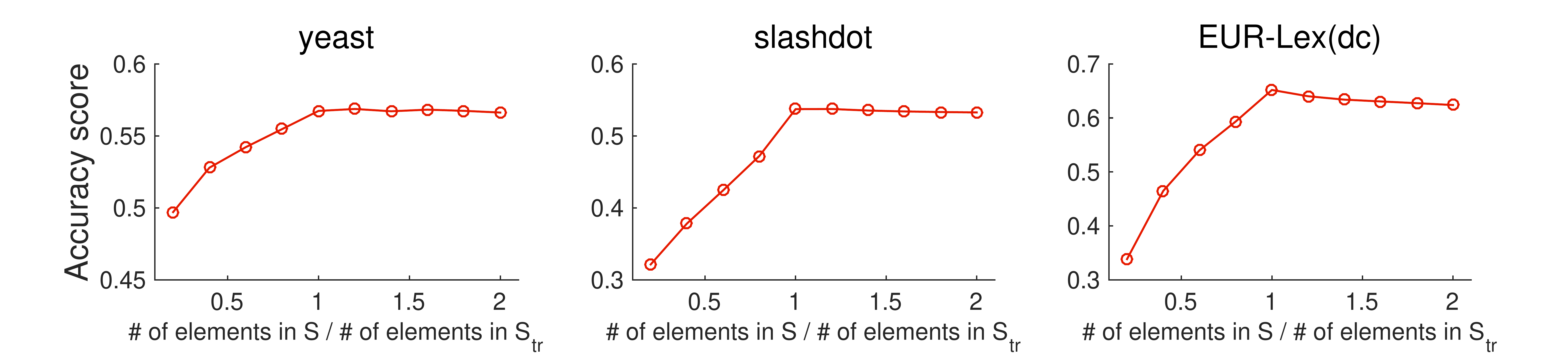}
	\caption{Accuracy score ($\uparrow$) of CLEMS with different size of candidate sets}
	\label{fig:ct_acc}
\end{figure}

\begin{figure}[!t]
	\centering
	\includegraphics[width=.95\textwidth]{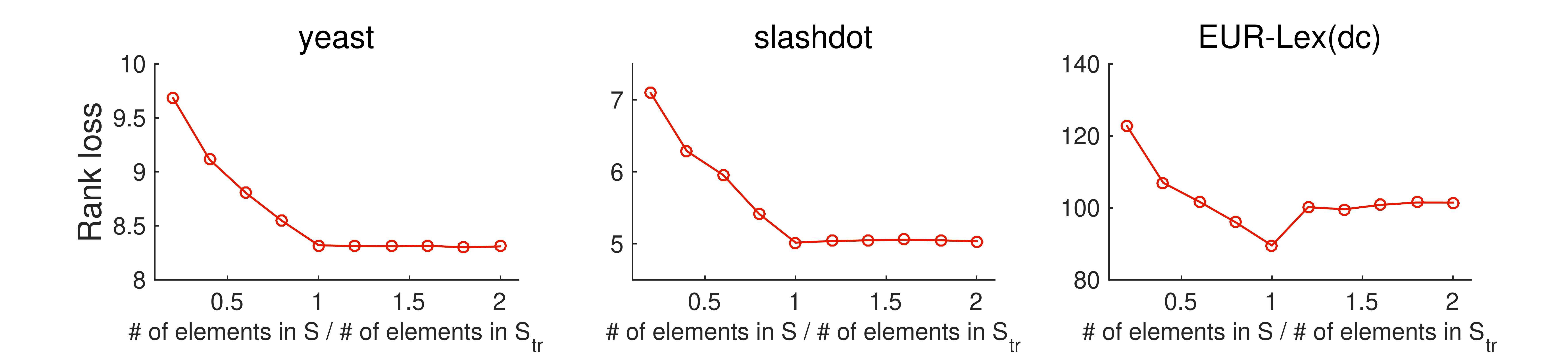}
	\caption{Rank loss ($\downarrow$) of CLEMS with different size of candidate sets}
	\label{fig:ct_rank}
\end{figure}

We conduct another experiment about the candidate set.
Instead of random sampling, we consider $\mathcal{S}_{all}$, which denotes the set of label vectors that appear in the training instances and the testing instances,
to estimate the benefit of ``peeping'' the testing label vectors and embedding them in advance.
We show the results of CLEMS with $\mathcal{S}_{tr}$ (CLEMS-train) and $\mathcal{S}_{all}$ (CLEMS-all) versus different embedded dimensions by Figure~\ref{fig:all_f1}, \ref{fig:all_acc}, and~\ref{fig:all_rank}.
From the figures, we see that the improvement of CLEMS-all over CLEMS-train is small and insignificant.
The results imply again that $\mathcal{S}_{tr}$ readily allows nearest-neighbor decoding to make sufficiently good choices.

\begin{figure}[!t]
	\centering
	\includegraphics[width=.95\textwidth]{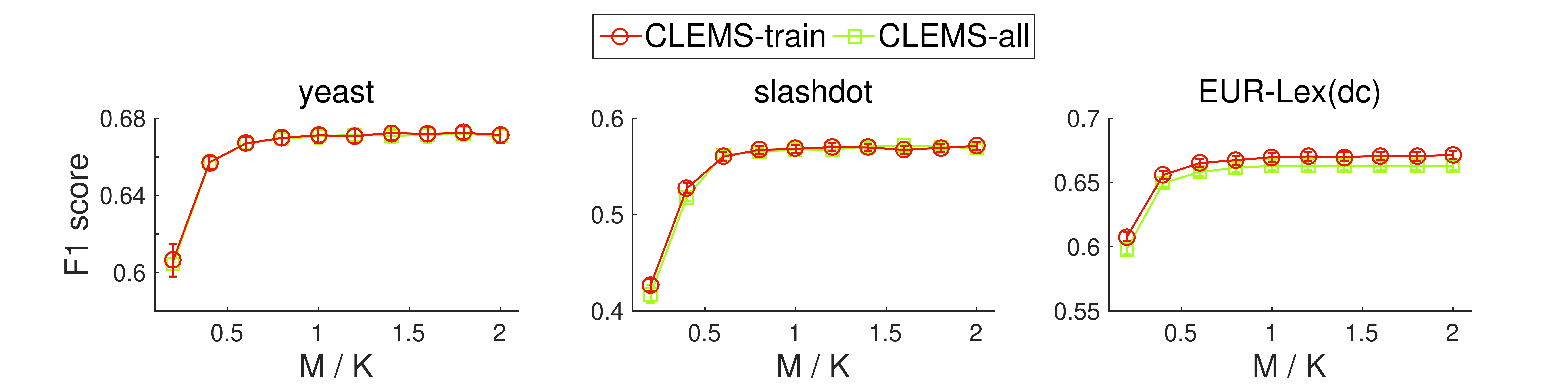}
	\caption{F1 score ($\uparrow$) with the 95\% confidence interval of CLEMS-train and CLEMS-all}
	\label{fig:all_f1}
\end{figure}

\begin{figure}[!t]
	\centering
	\includegraphics[width=.95\textwidth]{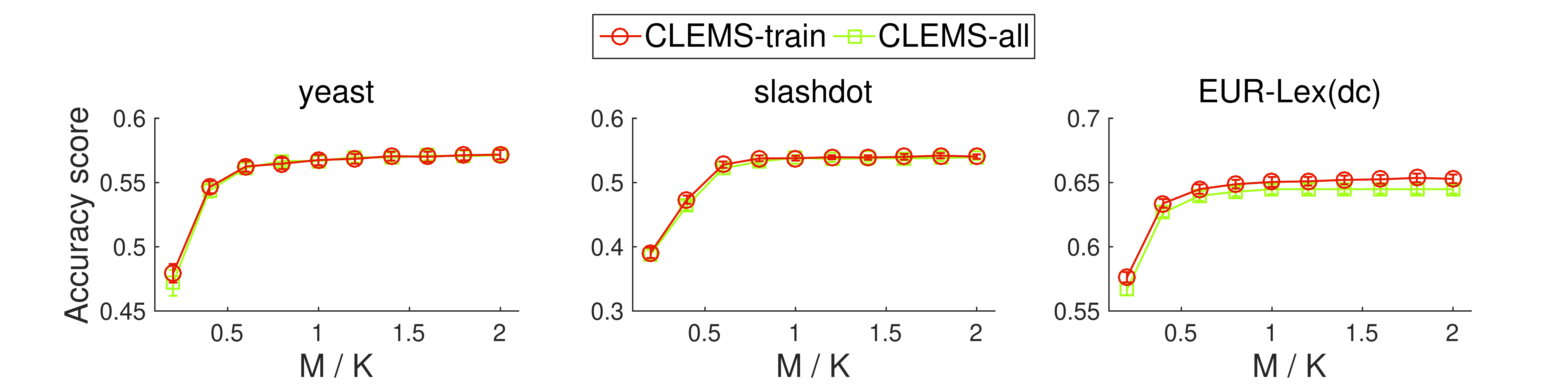}
	\caption{Accuracy score ($\uparrow$) with the 95\% confidence interval of CLEMS-train and CLEMS-all}
	\label{fig:all_acc}
\end{figure}

\begin{figure}[!t]
	\centering
	\includegraphics[width=.95\textwidth]{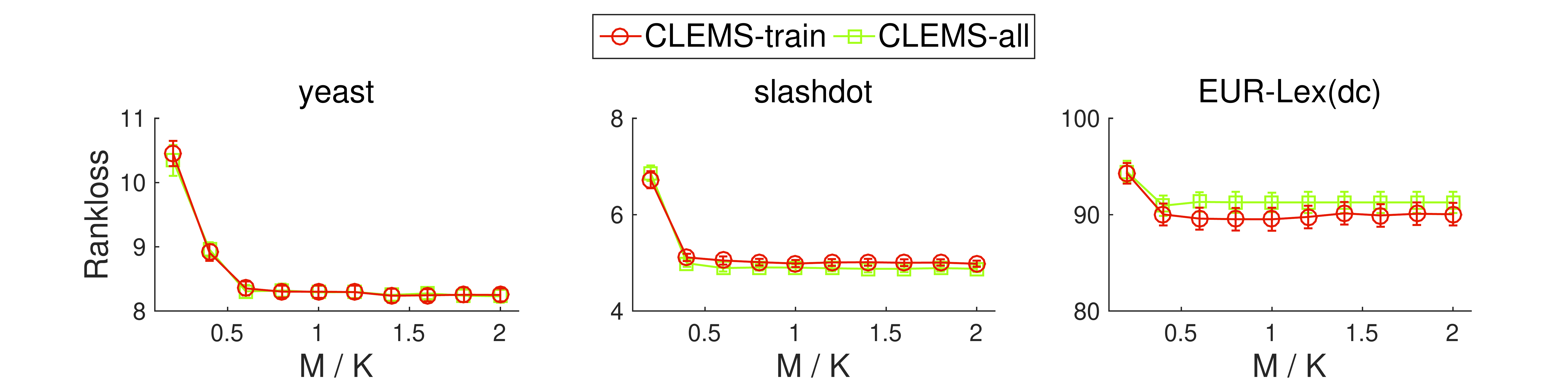}
	\caption{Rank loss ($\downarrow$) with the 95\% confidence interval of CLEMS-train and CLEMS-all}
	\label{fig:all_rank}
\end{figure}

Now, we discuss about the embedded dimension $M$.
From Figure~\ref{fig:all_f1}, \ref{fig:all_acc}, and~\ref{fig:all_rank}, CLEMS reaches better performance as $M$ increases.
For LSDR, $M$ plays an important role since it decides how much information can be preserved in the embedded space.
Nevertheless, For LSDE, the improvement becomes marginal when~$M$ increases.
The results suggest that for LSDE, the influence of the additional dimension is not large, and setting the embedded dimension $M = K$ is sufficiently good in practice. One possible reason for the sufficiency is that the criteria of interest
%The results also reveals a fact that most criteria
are generally not complicated enough and thus do not need more dimensions to preserve the cost information.

%%%%%%%%%%%%%%%%%%%%%%%%%%%%%%%%%%%%%%%% Cost-Sensitive %%%%%%%%%%%%%%%%%%%%%%%%%%%%%%%%%%%%%%%%
\subsection{Comparing CLEMS with cost-sensitive algorithms}
In this section, we compare CLEMS with two state-of-the-art cost-sensitive algorithms, probabilistic classifier chain (PCC) \citep{Dembczynski2010pcc,Dembczynski2011pcc2} and condensed filter tree (CFT) \citep{Li2014cft}.
Both CLEMS and CFT can handle arbitrary criteria while PCC can handle only those criteria with efficient inference rules.
In addition, we also report the results of some baseline algorithms, such as binary relevance~(BR) \citep{Tsoumakas2007br} and classifier chain~(CC) \citep{Read2011cc}.
Similar to previous experiments, the internal predictors of all algorithms are set as random forest \citep{Breiman2001rfr} implemented by scikit-learn \citep{Pedregosa2011sklearn} with the same parameter selection process.

%%%%%%%%%%%%%%%%%%%%%%%%%%%%%%%%%%%%%%%% Running Time %%%%%%%%%%%%%%%%%%%%%%%%%%%%%%%%%%%%%%%%
%\subsubsection{Comparison of running time}

\begin{figure}[!t]
	\centering
	\subfigure[average training time]{\includegraphics[width=.8\textwidth]{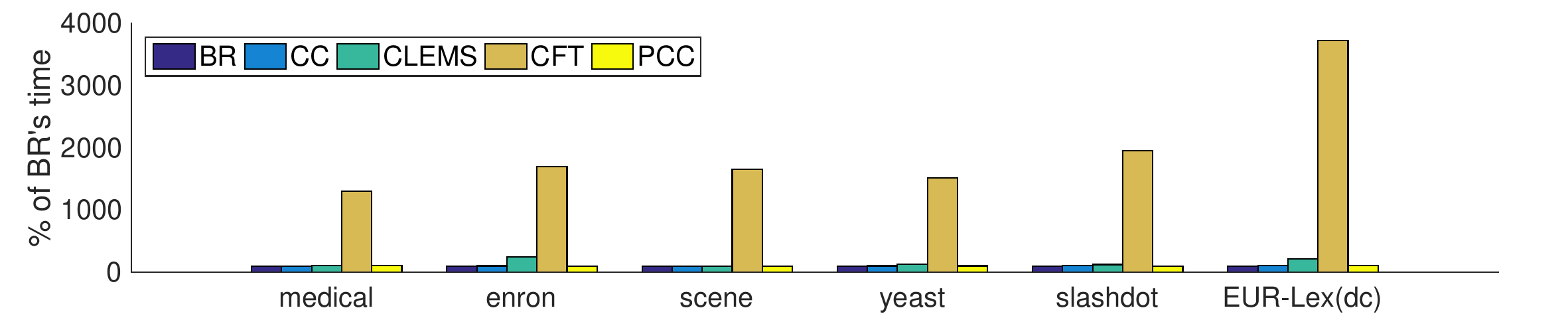}}
	\subfigure[average predicting time]{\includegraphics[width=.8\textwidth]{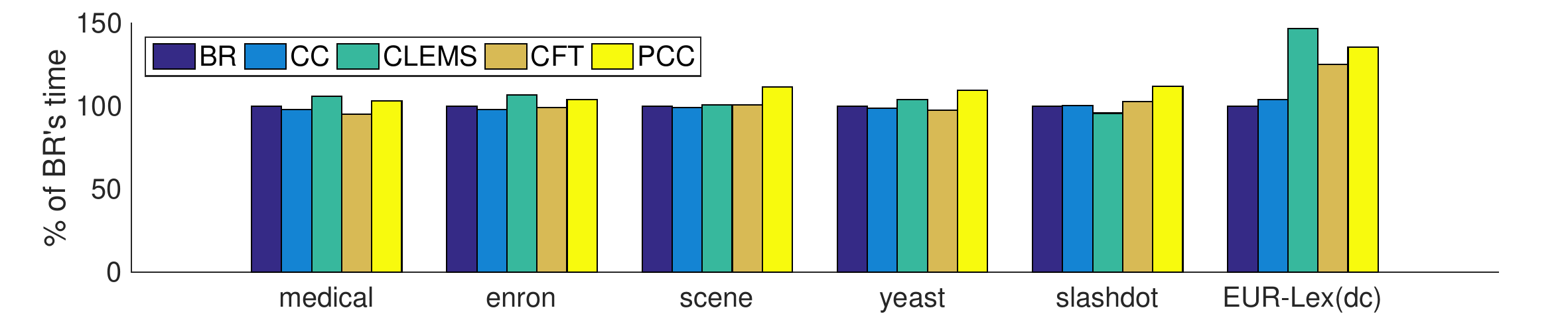}}
	\subfigure[average total running time]{\includegraphics[width=.8\textwidth]{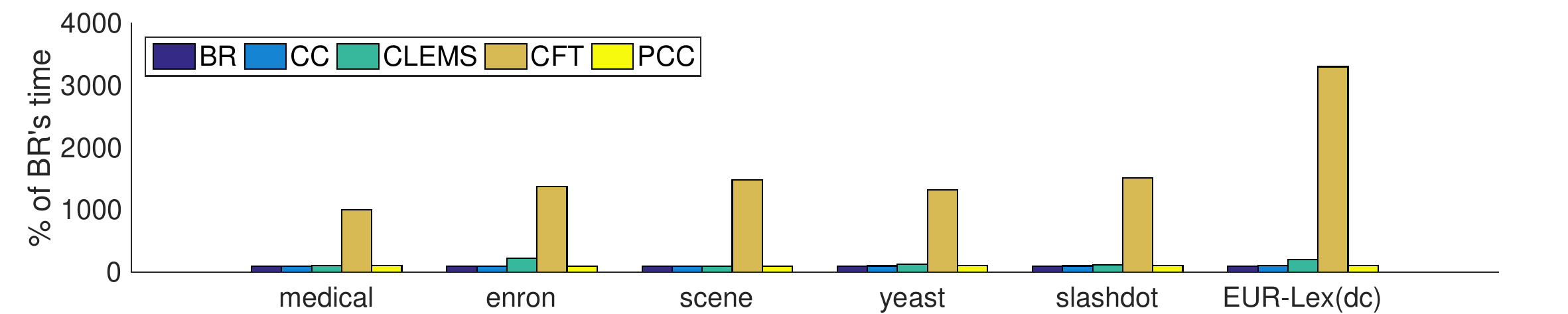}}
	\caption{Average running time when taking F1 score as cost function}
	\label{fig:time}
\end{figure}

\paragraph{Running Time.}
%We first compared the running time of theses algorithms.
Figure~\ref{fig:time} illustrates the average training, predicting, and total running time when taking F1 score as the intended criterion for the six largest datasets. The running time is normalized by the running time of BR.
For training time, CFT is the slowest, because it needs
to iteratively estimate the importance of each label and re-train internal predictors. CLEMS, which consumes time for MDS calculations, is intuitively slower than baseline algorithms and PCC during training, but still much faster than CFT.
For prediction time, all algorithms, including PCC (using inference calculation) and CLEMS (using nearest-neighbor calculation) are similarly fast. The results suggest that for CSMLC, CLEMS is superior to CFT and competitive to PCC for the overall efficiency.

%% We first notice that CFT takes much time on training.
%% This is because CFT needs to iteratively estimated the importance of each label and re-train internal predictors.
%% Then, we observe that while CLEMS needs to calculate the MDS-based embedding in training, the additional calculation does not take much time compared with CFT.
%% For the predicting time, although CLEMS needs to find the nearest neighbor and PCC needs the calculation on inference, they can still make predictions quickly.
%% When considering the total running time, we find that CFT is much slower than the others.
%% Both CLEMS and CFT can handle arbitrary cost functions, but CLEMS can run faster.
%% The results validate the efficiency of CLEMS.

%%%%%%%%%%%%%%%%%%%%%%%%%%%%%%%%%%%%%%%% Cost %%%%%%%%%%%%%%%%%%%%%%%%%%%%%%%%%%%%%%%%

\paragraph{Performance.}
We compare the performance of CLEMS and other algorithms across different criteria.
To demonstrate the full ability of CLEMS, in addition to F1 score, Accuracy score, and Rank loss, we further consider one additional criterion, \textit{Composition loss} = 1$+$5$\times$\textit{Hamming loss}$-$\textit{F1 score}, as used by \citet{Li2014cft}.
We also consider three more datasets (arts, flags, and language-log) that comes from other MLC works \citep{Tsoumakas2011mulan,Read2016meka}. 

\begin{table}[!t]
	\centering
	\caption{Performance across different criteria (mean $\pm$ ste (rank)) (best marked in bold)}
	\label{tab:cost}
	\tiny
	\begin{tabular}{cccccc}
		\hline
		Dataset & Alg. & F1 score ($\uparrow$) & Acc. score ($\uparrow$) & Rank loss ($\downarrow$) & Compo. loss ($\downarrow$) \\
		\hline
		\multirow{5}{*}{flags}  & BR    & $0.703 \pm 0.006\,(5)$	& $0.591 \pm 0.007\,(3)$	& $3.011 \pm 0.056\,(4)$	& $1.583 \pm 0.028\,(3)$	\\
		                        & CC    & $0.704 \pm 0.006\,(4)$	& $0.594 \pm 0.008\,(2)$	& $2.998 \pm 0.061\,(3)$	& $1.580 \pm 0.028\,(2)$	\\
		                        & CLEMS & $\bf0.731 \pm 0.005\,(1)$	& $\bf0.615 \pm 0.008\,(1)$	& $2.930 \pm 0.061\,(2)$	& $\bf1.575 \pm 0.026\,(1)$	\\
		                        & CFT   & $0.692 \pm 0.008\,(3)$	& $0.588 \pm 0.009\,(4)$	& $3.075 \pm 0.060\,(5)$	& $1.640 \pm 0.033\,(4)$ 	\\
		                        & PCC   & $0.706 \pm 0.006\,(2)$	& --						& $\bf2.857 \pm 0.051\,(1)$	& -- 						\\
		\hline
		\multirow{5}{*}{CAL.}   & BR    & $0.338 \pm 0.002\,(4)$	& $0.208 \pm 0.001\,(3)$	& $1504.8 \pm 7.98\,(4)$	& $\bf1.366 \pm 0.005\,(1)$	\\
		                        & CC    & $0.328 \pm 0.002\,(5)$	& $0.202 \pm 0.002\,(4)$	& $1520.9 \pm 9.04\,(5)$	& $1.371 \pm 0.006\,(2)$	\\
		                        & CLEMS & $\bf0.419 \pm 0.002\,(1)$	& $\bf0.273 \pm 0.002\,(1)$	& $1247.9 \pm 8.21\,(3)$	& $1.426 \pm 0.004\,(4)$ 	\\
		                        & CFT   & $0.371 \pm 0.003\,(3)$	& $0.237 \pm 0.002\,(2)$ 	& $1120.8 \pm 8.46\,(2)$ 	& $1.378 \pm 0.006\,(3)$	\\
		                        & PCC   & $0.391 \pm 0.002\,(2)$	& --						& $\bf993.6 \pm 4.75\,(1)$	& -- 						\\
		\hline
		\multirow{5}{*}{birds}  & BR    & $0.569 \pm 0.007\,(5)$	& $0.551 \pm 0.007\,(4)$	& $6.845 \pm 0.139\,(5)$	& $0.656 \pm 0.011\,(4)$	\\
		                        & CC    & $0.570 \pm 0.007\,(4)$	& $0.552 \pm 0.007\,(3)$	& $6.825 \pm 0.138\,(4)$	& $0.654 \pm 0.011\,(3)$	\\
		                        & CLEMS & $\bf0.677 \pm 0.006\,(1)$	& $\bf0.642 \pm 0.007\,(1)$	& $4.886 \pm 0.142\,(2)$	& $\bf0.563 \pm 0.012\,(1)$	\\
		                        & CFT   & $0.601 \pm 0.007\,(3)$	& $0.586 \pm 0.007\,(2)$	& $4.908 \pm 0.148\,(3)$	& $0.607 \pm 0.012\,(2)$ 	\\
		                        & PCC   & $0.636 \pm 0.007\,(2)$	& --						& $\bf3.660 \pm 0.103\,(1)$	& -- 						\\
		\hline
		\multirow{5}{*}{emot.}	& BR    & $0.596 \pm 0.005\,(5)$	& $0.523 \pm 0.004\,(4)$	& $1.764 \pm 0.022\,(5)$	& $1.352 \pm 0.012\,(4)$	\\
								& CC    & $0.615 \pm 0.005\,(4)$	& $0.539 \pm 0.004\,(3)$	& $1.715 \pm 0.021\,(4)$	& $1.329 \pm 0.013\,(3)$ 	\\
								& CLEMS & $\bf0.676 \pm 0.005\,(1)$	& $\bf0.589 \pm 0.006\,(1)$	& $1.484 \pm 0.020\,(2)$	& $\bf1.271 \pm 0.013\,(1)$	\\
		                        & CFT   & $0.640 \pm 0.004\,(3)$	& $0.557 \pm 0.004\,(2)$	& $1.563 \pm 0.018\,(3)$	& $1.324 \pm 0.016\,(2)$ 	\\
		                        & PCC   & $0.643 \pm 0.005\,(2)$	& --						& $\bf1.467 \pm 0.018\,(1)$	& -- 						\\
		\hline
		\multirow{5}{*}{medic.} & BR    & $0.517 \pm 0.006\,(5)$	& $0.496 \pm 0.006\,(4)$	& $13.784 \pm 0.175\,(5)$	& $0.562 \pm 0.006\,(4)$	\\
		                        & CC    & $0.533 \pm 0.006\,(4)$	& $0.512 \pm 0.006\,(3)$	& $13.328 \pm 0.167\,(4)$	& $0.544 \pm 0.007\,(3)$	\\
		                        & CLEMS & $\bf0.814 \pm 0.004\,(1)$	& $\bf0.786 \pm 0.004\,(1)$	& $5.170 \pm 0.159\,(2)$	& $\bf0.289 \pm 0.005\,(1)$	\\
		                        & CFT   & $0.635 \pm 0.005\,(2)$	& $0.613 \pm 0.005\,(2)$	& $5.811 \pm 0.131\,(3)$	& $0.438 \pm 0.007\,(2)$ 	\\
		                        & PCC   & $0.573 \pm 0.006\,(3)$	& --						& $\bf4.234 \pm 0.109\,(1)$	& -- 						\\
		\hline
		\multirow{5}{*}{lang.}  & BR    & $0.160 \pm 0.004\,(5)$	& $0.159 \pm 0.004\,(4)$	& $42.46 \pm 0.271\,(5)$	& $0.919 \pm 0.004\,(4)$	\\
		                        & CC    & $0.161 \pm 0.004\,(4)$	& $0.160 \pm 0.004\,(3)$	& $42.42 \pm 0.168\,(4)$	& $0.918 \pm 0.004\,(3)$	\\
		                        & CLEMS & $\bf0.375 \pm 0.005\,(1)$	& $\bf0.327 \pm 0.005\,(1)$	& $31.03 \pm 0.383\,(2)$	& $\bf0.734 \pm 0.007\,(1)$	\\
		                        & CFT   & $0.168 \pm 0.004\,(3)$	& $0.164 \pm 0.004\,(2)$	& $34.16 \pm 0.285\,(3)$	& $0.910 \pm 0.005\,(2)$ 	\\
		                        & PCC   & $0.247 \pm 0.004\,(2)$	& --						& $\bf19.11 \pm 0.211\,(1)$	& -- 						\\
		\hline
		\multirow{5}{*}{enron}  & BR    & $0.543 \pm 0.003\,(4)$	& $0.433 \pm 0.003\,(4)$	& $44.83 \pm 0.376\,(5)$	& $0.688 \pm 0.004\,(4)$	\\
		                        & CC    & $0.553 \pm 0.003\,(3)$	& $0.443 \pm 0.003\,(3)$	& $43.82 \pm 0.429\,(4)$	& $0.678 \pm 0.005\,(3)$	\\
		                        & CLEMS & $\bf0.606 \pm 0.003\,(1)$	& $\bf0.491 \pm 0.004\,(1)$	& $29.40 \pm 0.300\,(3)$	& $\bf0.659 \pm 0.005\,(1)$	\\
		                        & CFT   & $0.557 \pm 0.004\,(2)$	& $0.448 \pm 0.003\,(2)$	& $26.64 \pm 0.311\,(2)$	& $0.677 \pm 0.005\,(2)$ 	\\
		                        & PCC   & $0.542 \pm 0.003\,(5)$	& --						& $\bf25.11 \pm 0.263\,(1)$	& --						\\
		\hline
		\multirow{5}{*}{scene}  & BR    & $0.577 \pm 0.003\,(5)$	& $0.568 \pm 0.004\,(4)$	& $1.169 \pm 0.010\,(5)$	& $0.866 \pm 0.007\,(4)$	\\
		                        & CC    & $0.598 \pm 0.004\,(4)$	& $0.590 \pm 0.004\,(3)$	& $1.122 \pm 0.012\,(4)$	& $0.833 \pm 0.009\,(3)$	\\
		                        & CLEMS & $\bf0.770 \pm 0.003\,(1)$	& $\bf0.760 \pm 0.004\,(1)$	& $0.672 \pm 0.015\,(2)$	& $\bf0.578 \pm 0.009\,(1)$	\\
		                        & CFT   & $0.703 \pm 0.004\,(3)$	& $0.656 \pm 0.004\,(2)$	& $0.723 \pm 0.011\,(3)$	& $0.776 \pm 0.009\,(2)$ 	\\
		                        & PCC   & $0.745 \pm 0.003\,(2)$	& --						& $\bf0.645 \pm 0.005\,(1)$	& -- 						\\
		\hline
		\multirow{5}{*}{yeast}  & BR    & $0.611 \pm 0.002\,(5)$	& $0.503 \pm 0.002\,(4)$	& $9.673 \pm 0.048\,(5)$	& $1.345 \pm 0.006\,(3)$	\\
		                        & CC    & $0.612 \pm 0.003\,(4)$	& $0.512 \pm 0.003\,(3)$	& $9.530 \pm 0.067\,(4)$	& $1.352 \pm 0.009\,(4)$	\\
		                        & CLEMS & $\bf0.671 \pm 0.002\,(1)$	& $\bf0.568 \pm 0.002\,(1)$	& $\bf8.302 \pm 0.049\,(1)$	& $\bf1.308 \pm 0.006\,(1)$	\\
		                        & CFT   & $0.649 \pm 0.002\,(2)$	& $0.543 \pm 0.002\,(2)$	& $8.566 \pm 0.052\,(3)$	& $1.335 \pm 0.007\,(2)$ 	\\
		                        & PCC   & $0.614 \pm 0.002\,(3)$	& --						& $8.469 \pm 0.057\,(2)$	& -- 						\\
		\hline
		\multirow{5}{*}{slash.} & BR    & $0.215 \pm 0.002\,(5)$	& $0.208 \pm 0.002\,(4)$	& $9.819 \pm 0.030\,(5)$	& $1.007 \pm 0.003\,(4)$	\\
		                        & CC    & $0.230 \pm 0.002\,(4)$	& $0.222 \pm 0.002\,(3)$	& $9.662 \pm 0.027\,(4)$	& $0.990 \pm 0.003\,(3)$	\\
		                        & CLEMS & $\bf0.568 \pm 0.002\,(1)$	& $\bf0.538 \pm 0.002\,(1)$	& $4.986 \pm 0.038\,(2)$	& $\bf0.668 \pm 0.003\,(1)$	\\
		                        & CFT   & $0.429 \pm 0.003\,(3)$	& $0.402 \pm 0.003\,(2)$	& $5.677 \pm 0.033\,(3)$	& $0.798 \pm 0.003\,(2)$ 	\\
		                        & PCC   & $0.503 \pm 0.003\,(2)$	& --						& $\bf4.472 \pm 0.029\,(1)$	& -- 						\\
		\hline
		\multirow{5}{*}{arts}   & BR    & $0.167 \pm 0.002\,(5)$	& $0.156 \pm 0.002\,(4)$	& $17.221 \pm 0.064\,(5)$	& $1.117 \pm 0.003\,(4)$	\\
		                        & CC    & $0.170 \pm 0.002\,(4)$	& $0.160 \pm 0.002\,(3)$	& $17.173 \pm 0.064\,(4)$	& $1.113 \pm 0.003\,(3)$	\\
		                        & CLEMS & $\bf0.492 \pm 0.002\,(1)$	& $\bf0.451 \pm 0.003\,(1)$	& $9.865 \pm 0.079\,(2)$	& $\bf0.815 \pm 0.006\,(1)$	\\
		                        & CFT   & $0.334 \pm 0.002\,(3)$	& $0.281 \pm 0.002\,(2)$	& $10.071 \pm 0.060\,(3)$	& $1.001 \pm 0.003\,(2)$ 	\\
		                        & PCC   & $0.349 \pm 0.002\,(2)$	& --						& $\bf8.467 \pm 0.047\,(1)$	& -- 						\\
		\hline
		\multirow{5}{*}{EUR.}   & BR    & $0.417 \pm 0.002\,(4)$	& $0.411 \pm 0.001\,(3)$	& $168.38 \pm 0.63\,(4)$	& $0.593 \pm 0.002\,(3)$	\\
		                        & CC    & $0.416 \pm 0.002\,(5)$	& $0.410 \pm 0.001\,(4)$	& $168.57 \pm 0.61\,(5)$	& $0.594 \pm 0.002\,(4)$	\\
		                        & CLEMS & $\bf0.670 \pm 0.002\,(1)$	& $\bf0.650 \pm 0.002\,(1)$	& $89.52 \pm 0.61\,(2)$		& $\bf0.344 \pm 0.002\,(1)$	\\
		                        & CFT   & $0.456 \pm 0.002\,(3)$	& $0.450 \pm 0.002\,(2)$	& $129.53 \pm 0.75\,(3)$	& $0.552 \pm 0.002\,(2)$ 	\\
		                        & PCC   & $0.483 \pm 0.002\,(2)$	& --						& $\bf43.28 \pm 0.22\,(1)$	& -- 						\\
		\hline
	\end{tabular}
\end{table}

The results are shown by Table~\ref{tab:cost}.
Accuracy score and Composition loss for PCC are left blank since there is no efficient inference rules.
The first finding is that cost-sensitive algorithms (CLEMS, PCC, and CFT) generally perform better than non-cost-sensitive algorithms (BR and CC) across different criteria.
This validates the usefulness of cost-sensitivity for MLC algorithms.

For F1 score, Accuracy score, and Composition loss, CLEMS outperforms PCC and CFT in most cases.
The reason is that these criteria evaluate all the labels jointly, and CLEMS can globally locate the hidden structure of labels to facilitate more effective learning, while PCC and CFT are chain-based algorithms and only locally discover the relation between labels.
For Rank loss, PCC performs the best in most cases.
One possible reason is that Rank loss can be expressed
as a special weighted Hamming loss that does not require
globally locating the hidden structure.
Thus, chaining algorithms like PCC can still perform decently. Note, however, that CLEMS is often the second best
for Rank loss as well.

%and thus will evaluate the labels independently.
%In this case, locally discovering the relation between labels does not lose much label information and it results in that PCC perform well.
%But we notice that CLEMS still has better performance than CFT.

In summary, we identify two merits of CLEMS.
The first is that while PCC performs better on Rank loss, CLEMS is competitive for general cost-sensitivity and can be coupled with arbitrary criteria.
The second is that although CFT also shoots for general cost-sensitivity, CLEMS outperforms CFT in most cases for all criteria. The results make CLEMS a decent first-hand-choice for general CSMLC.

\paragraph{Performance on other criteria.}

So far, we have justified the benefits of CLEMS for directly optimizing towards the criterion of interest. Next, we discuss about whether CLEMS can be used to \textit{indirectly} optimize other criteria of interest, particularly when the criterion cannot be meaningfully expressed as the input to CLEMS. CLEMS follows the setting in Section~\ref{sec:csle} to accept \textit{example-based} criterion, which works on one label vector at a time. A more general type of criteria considers multiple or all the label vectors at the same time, called \textit{label-based} criteria. Two representative \textit{label-based} criteria are Micro F1 and Macro F1 \citep{Madjarov2012score2}, and will be studied next.
The former calculates the F1 score over all the label components of testing examples, and the latter averages the per-label F1 score across examples.
To the best of our knowledge, there is no cost-sensitive algorithms can handle arbitrary \textit{label-based} criteria. 

Another criterion that we will study is subset accuracy \citep{Madjarov2012score2}. It can be expressed as an \textit{example-based} criterion with two possible values: whether the label vector is completely correct or not. The criterion is very strict and does not come with trade-off on big or small prediction errors. Thus, it is generally not meaningful to feed the criterion directly to CLEMS or other CSMLC algorithms.

Next, we demonstrate how CLEMS can indirectly optimize Micro/Macro F1 score and subset accuracy when fed with other criteria as inputs.
%Although the existing cost-sensitive algorithms cannot achieve cost-sensitivity to the above-mentioned criteria, in this experiment, we want to know whether or not we can choose the proper criterion as the input for the cost-sensitive algorithms such that the algorithms can perform well on these criteria.
We consider 6 pre-divided datasets (emotions, scene, yeast, medical, enron, and Corel5k) as used by \citet{Madjarov2012score2}.
We consider two baseline algorithms (BR and CC), CLEMS with three different input criteria (F1 score, Accuracy score, Rank loss), and PCC with two different criteria (F1 score and Rank loss) that come with efficient inference rules.
The results are shown in Table~\ref{tab:label}.

\begin{table}[!t]
	\centering
	\caption{Comparison for other criteria (best marked in bold)}
	\label{tab:label}
	\tiny
	\begin{tabular}{ccccccccc}
		\hline
		\multirow{2}{*}{Dataset} & \multirow{2}{*}{Criterion} & \multirow{2}{*}{BR} & \multirow{2}{*}{CC} & CLEMS& CLEMS & CLEMS & PCC & PCC \\
		& & & & (F1) & (Acc.) & (Rank.) & (F1) & (Rank.) \\
		\hline
		\multirow{3}{*}{emotions}	& Macro F1 ($\uparrow$) & $0.673$ & $0.682$ & $0.703$ & $\bf0.711$ & $0.708$ & $0.700$ & $0.698$ \\
									& Micro F1 ($\uparrow$) & $0.672$ & $0.708$ & $0.705$ & $\bf0.717$ & $0.708$ & $0.679$ & $0.700$ \\
									& Subset Acc. ($\uparrow$)& $0.282$ & $0.317$ & $0.248$ & $\bf0.337$ & $0.262$ & $0.277$ & $0.248$ \\
		\hline
		\multirow{3}{*}{medical}	& Macro F1 ($\uparrow$) & $0.345$ & $0.364$ & $0.395$ & $\bf0.424$ & $0.408$ & $0.361$ & $0.188$ \\
									& Micro F1 ($\uparrow$) & $0.692$ & $0.699$ & $0.757$ & $\bf0.788$ & $0.657$ & $0.586$ & $0.316$ \\
									& Subset Acc. ($\uparrow$)& $0.490$ & $0.511$ & $0.598$ & $\bf0.673$ & $0.393$ & $0.348$ & $0.024$ \\
		\hline
		\multirow{3}{*}{enron}		& Macro F1 ($\uparrow$) & $0.105$ & $0.114$ & $0.163$ & $0.148$ & $\bf0.223$ & $0.214$ & $0.222$ \\
									& Micro F1 ($\uparrow$) & $0.453$ & $0.477$ & $\bf0.559$ & $0.531$ & $0.501$ & $0.511$ & $0.420$ \\
									& Subset Acc. ($\uparrow$)& $0.033$ & $0.036$ & $\bf0.073$ & $\bf0.073$ & $0.007$ & $0.029$ & $0.000$ \\
		\hline
		\multirow{3}{*}{scene}		& Macro F1 ($\uparrow$) & $0.693$ & $0.682$ & $0.766$ & $\bf0.783$ & $0.766$ & $0.748$ & $0.698$ \\
									& Micro F1 ($\uparrow$) & $0.690$ & $0.705$ & $0.736$ & $\bf0.775$ & $0.733$ & $0.727$ & $0.700$ \\
									& Subset Acc. ($\uparrow$)& $0.533$ & $0.560$ & $0.520$ & $\bf0.704$ & $0.526$ & $0.528$ & $0.278$ \\
		\hline
		\multirow{3}{*}{yeast}		& Macro F1 ($\uparrow$) & $0.450$ & $0.458$ & $0.496$ & $0.492$ & $\bf0.506$ & $0.464$ & $0.474$ \\
									& Micro F1 ($\uparrow$) & $0.641$ & $0.644$ & $0.674$ & $\bf0.682$ & $0.672$ & $0.605$ & $0.665$ \\
									& Subset. Acc. ($\uparrow$)& $0.164$ & $\bf0.220$ & $0.150$ & $0.214$ & $0.127$ & $0.121$ & $0.132$ \\
		\hline
		\multirow{3}{*}{Corel5k}	& Macro F1 ($\uparrow$) & $0.038$ & $0.055$ & $0.073$ & $\bf0.092$ & $0.069$ & $0.026$ & $0.042$ \\
									& Micro F1 ($\uparrow$) & $0.072$ & $0.071$ & $0.264$ & $\bf0.272$ & $0.255$ & $0.197$ & $0.091$ \\
									& Subset. Acc. ($\uparrow$)& $0.000$ & $0.002$ & $0.020$ & $\bf0.034$ & $0.020$ & $0.000$ & $0.000$ \\
		\hline

	\end{tabular}
\end{table}

From the table, we observe that when selecting a proper criterion as the input of CSMLC algorithms (CLEMS or PCC), they
can readily perform better than the baseline algorithms.
The results justify the value of the CSMLC algorithms beyond handling \textit{example-based} criteria.
In particular, the cost input to CSMLC algorithms act as a tunable parameter towards optimizing other true criteria of interests.
We also observe that CLEMS, especially CLEMS-Acc, performs better on the three criteria than PCC in the most datasets, which
again validate the usefulness of CLEMS. An interesting future direction is whether CLEMS can be further extended to
achieve cost-sensitivity for \textit{label-based} criteria.

\section{Conclusion}
\label{sec:final}
We propose a novel cost-sensitive label embedding algorithm called cost-sensitive label embedding with multidimensional scaling (CLEMS). CLEMS successfully embeds the label information and cost information into an arbitrary-dimensional hidden structure by the classic multidimensional scaling approach for manifold learning, and handles asymmetric cost functions with our careful design of the mirroring trick.
With the embedding, CLEMS can make cost-sensitive predictions efficiently and effectively by decoding to the nearest neighbor within a proper candidate set. The empirical results demonstrate that CLEMS is superior to state-of-the-art label embedding algorithms across different cost functions. To the best of our knowledge, CLEMS is the very first algorithm that achieves cost-sensitivity within label embedding, and opens a promising future research direction of designing cost-sensitive label embedding algorithms using manifold learning approaches.

\begin{acknowledgements}
We thank the anonymous reviewers for valuable suggestions. This material is based upon work supported by the
Air Force Office of Scientific Research, Asian Office of Aerospace Research and Development (AOARD) under award
number FA2386-15-1-4012, and by the Ministry of Science and Technology of Taiwan under number MOST 103-2221-E-002-149-MY3.
\end{acknowledgements}

\bibliographystyle{spbasic}      % basic style, author-year citations
\bibliography{ecml2017clems}

\end{document}